\documentclass[sigconf,nonacm]{acmart}
\pdfoutput=1
 
\usepackage{amsmath}
\usepackage{xcolor}
\usepackage{xspace} 
\usepackage{paralist} 

\usepackage{booktabs}  
\usepackage{csvsimple} 

\usepackage{caption} 
\usepackage{subcaption} 

\usepackage{bm}

\usepackage{tikz}
\usepackage{pgf}
\usepackage{pgfplots}
\usepackage{multirow}

\usepackage{circledsteps}
\tikzset{/csteps/outer color=blue}

\usetikzlibrary{positioning}

\usepackage[normalem]{ulem}

\usepackage{xpatch}
\usepackage{textcase}
\makeatletter
\xpatchcmd{\@sect}{\uppercase}{\MakeTextUppercase}{}{}
\xpatchcmd{\@sect}{\uppercase}{\MakeTextUppercase}{}{}
\makeatother

\usepackage[explicit]{titlesec}

\titleformat{\paragraph}[runin]
{\normalfont\bfseries}
{}
{0pt}
{#1}

\titlespacing{\paragraph}
{0pt} 
{1ex plus .1ex minus .2ex}
{1ex} 

\usepackage[appendix=append,bibliography=common]{apxproof} 

\theoremstyle{plain}
\newtheoremrep{theorem}{Theorem}[section]
\newtheoremrep{lemma}[theorem]{Lemma}
\newtheoremrep{proposition}[theorem]{Proposition}
\newtheoremrep{claim}[theorem]{Claim}
\newtheoremrep{corollary}[theorem]{Corollary}
\newtheoremrep{observation}[theorem]{Observation}
\newtheoremrep{remark}[theorem]{Remark}

\newcommand{\natnum}{\mathbb{N}}
\newcommand{\nats}{\mathbb{N}}
\newcommand{\natsinf}{\mathbb{N}_{\infty}}
\newcommand{\reals}{\mathbb{R}}

\newcommand{\multileft}{\ensuremath{\{\!\!\{}}
\newcommand{\multiright}{\ensuremath{\}\!\!\}}}
\newcommand{\mset}[1]{\ensuremath{\multileft #1 \multiright}}
\DeclareMathOperator{\supp}{\textit{supp}}
\newcommand{\mrestr}[2]{\ensuremath{{#1}|_{\leq #2}}}
\newcommand{\restr}[2]{\ensuremath{{#1}|_{#2}}}

\DeclareMathOperator{\inc}{\mathsf{in}} 
\DeclareMathOperator{\incidence}{\mathsf{incidence}} 

\newcommand{\gnn}{GNN\xspace} 
\newcommand{\gnns}{GNNs\xspace} 

\newcommand{\agg}{\textsc{Agg}} 
\newcommand{\comb}{\textsc{Comb}} 

\newcommand{\size}[1]{\ensuremath{|#1|}}
\newcommand{\seq}[1]{\ensuremath{\overline{#1}}}
\renewcommand{\L}{\ensuremath{\seq L}}

\newcommand{\cls}[1]{\ensuremath{\mathcal{#1}}}
\newcommand{\ctrans}{\ensuremath{\cls C}} 

\newcommand{\prob}{\ensuremath{\mathcal{P}}}
\newcommand{\topol}{\ensuremath{\mathcal{S}}}

\DeclareMathOperator{\loss}{\textsc{Loss}}
\newcommand{\train}{\ensuremath{T}}

\DeclareMathOperator{\colr}{\mathsf{cr}} 

\newcommand{\indist}[1]{\ensuremath{\sim_{#1}}}
\newcommand{\class}[2]{\ensuremath{[#1]_{#2}}}

\begin{document}

\title{Learning Graph Neural Networks using Exact Compression}

\author{Jeroen Bollen}
\orcid{0000-0002-8881-5241}
\email{jeroen.bollen@uhasselt.be}
\affiliation{%
  \institution{UHasselt, Data Science Institute}
  \country{Belgium}
}

\author{Jasper Steegmans}
\orcid{0000-0003-2087-9430}
\email{jasper.steegmans@uhasselt.be}
\affiliation{%
  \institution{UHasselt, Data Science Institute}
  \country{Belgium}
}

\author{Jan Van den Bussche}
\email{jan.vandenbussche@uhasselt.be}
 \orcid{0000-0003-0072-3252}
\affiliation{%
  \institution{UHasselt, Data Science Institute}
  \country{Belgium}
}

\author{Stijn Vansummeren}
\email{stijn.vansummeren@uhasselt.be}
\orcid{0000-0001-7793-9049}
\affiliation{%
  \institution{UHasselt, Data Science Institute}
  \country{Belgium}
}

\begin{abstract}
  Graph Neural Networks (\gnns) are a form of deep learning that enable a wide
  range of machine learning applications on graph-structured data. The learning
  of \gnns, however, is known to pose challenges for memory-constrained devices
  such as GPUs. In this paper, we study \emph{exact compression} as a way to
  reduce the memory requirements of learning \gnns on large graphs. In
  particular, we adopt a formal approach to compression and propose a
  methodology that transforms \gnn learning problems into provably equivalent
  compressed \gnn learning problems.  In a
  preliminary experimental evaluation, we give insights into the compression
  ratios that can be obtained on real-world graphs and apply our methodology to
  an existing \gnn benchmark.
\end{abstract}

\maketitle

\section{Introduction}
\label{sec:intro}
Whereas Machine Learning (ML) has traditionally been most successful in
analyzing traditional unstructured data such as text or images, ML over
structured data such as graphs has become an active area of research in
the past decade.  In particular, Graph Neural Networks (\gnns for short) are a
form of deep learning architectures that enable a wide range of ML applications
on graph-structured data, such as molecule classification, knowledge graph
completion, and web-scale recommendations~\cite{battagliaRelationalInductiveBiases2018,gilmerNeuralMessagePassing2017,hamiltonInductiveRepresentationLearning2017,yingGraphConvolutionalNeural2018,hamiltonGraphRepresentationLearning2020}. 
At their core, \gnns allow to embed graph nodes into vector space. Crucially, the obtained vectors can capture graph structure, 
which is essential for the ML
applications already cited.

While \gnns are hence an attractive mechanism for ML on
graphs, 
learning \gnns is known to be resource-demanding which limits their
scalability~\cite{wuComprehensiveSurveyGraph2021,zhuSpikingGraphConvolutional2022}. In
particular, for large graphs it becomes difficult to encode all the required
information into the limited memory of hardware accelerators like GPUs. For this
reason, scalable methods for learning \gnns on large graphs are an active
subject of research (e.g.,
\cite{wangNeutronStarDistributedGNN2022,zhengByteGNNEfficientGraph2022,yuanDistributedLearningFully2022,linCharacterizingUnderstandingDistributed2022,liaoSCARAScalableGraph2022,pengSancusStalenessawareCommunicationavoiding2022,
  generaleScalingRGCNTraining2022,salhaDegeneracyFrameworkScalable2019,dengGraphZoomMultilevelSpectral2020a,huangAdaptiveSamplingFast2018,chenFastGCNFastLearning2018,hamiltonInductiveRepresentationLearning2017}). Broadly
speaking, we can identify three different principles for obtaining scalability
in the literature: (1) distributing computation across multiple machines or
GPUs~\cite{wangNeutronStarDistributedGNN2022,zhengByteGNNEfficientGraph2022,yuanDistributedLearningFully2022,linCharacterizingUnderstandingDistributed2022,liaoSCARAScalableGraph2022,pengSancusStalenessawareCommunicationavoiding2022,
  generaleScalingRGCNTraining2022}; (2) learning on a sample of the input graph
instead of the entire
graph~\cite{,huangAdaptiveSamplingFast2018,chenFastGCNFastLearning2018,hamiltonInductiveRepresentationLearning2017};
and (3)
compression~\cite{generaleScalingRGCNTraining2022,salhaDegeneracyFrameworkScalable2019,dengGraphZoomMultilevelSpectral2020a,liangMILEMultiLevelFramework2021}.
Compression-based approaches
limit the memory requirements of learning \gnns on large graphs by reducing the
input graph into a smaller graph and then learn on this smaller, reduced graph
instead. In this paper, we are concerned with compression.

Compression methods are based on collapsing multiple input nodes into a single
reduced node in the compressed graph. Methods vary, however, in how they
collapse nodes. For example, Deng et
al.~\cite{dengGraphZoomMultilevelSpectral2020a} use spectral analysis for this
purpose; Liang et al.~\cite{liangMILEMultiLevelFramework2021} use variants of
multi-level graph partitioning; and Generale et
al.~\cite{generaleScalingRGCNTraining2022}, who specifically consider knowledge
graphs, use general heuristics (such as two nodes having equal set of
attributes) or bisimulation. While these methods give intuitive reasons to argue that the
structure of the obtained compressed graph should be similar to that of the
original graph, no formal guarantee is ever given that learning on the
compressed graph is in any way equivalent to learning on the original
graph. Furthermore, the methods are usually devised and tested for a specific
\gnn architecture (such as Graph Convolutional Networks, GCN). It is therefore
unclear how they fare on other \gnn architectures. Inherently, these methods are
hence heuristics. At best the compressed graphs that they generate
\emph{approximate} the original graph structure, and it is difficult to predict
for which \gnn architectures this approximation is good enough, and for which
architectures it poses a problem.

Towards a more principled study of learning \gnns on compressed
graphs, we propose to take a formal approach and study
\emph{exact compression} instead. We make the following
contributions.

(1.)  We formally define when two learning problems involving graph neural
  networks are equivalent. Based on this definition, our goal is to transform a
  given problem into a smaller, equivalent problem based on
  compression. (Section~\ref{sec:prelims}.)

(2.)  We develop a compression methodology that is guaranteed to always yield an
  equivalent learning problem and that is applicable to a wide class of \gnn
  architectures known as \emph{aggregate-combine
\gnns}~\cite{groheLogicGraphNeural2021,barceloExpressivePowerGraph2020,hamiltonGraphRepresentationLearning2020}. This
  class includes all Graph Convolutional
  Networks~\cite{hamiltonGraphRepresentationLearning2020}. Our methodology is
  based on recent insights into the expressiveness of aggregate-combine \gnns
  ~\cite{morrisWeisfeilerLemanGo2019,xuHowPowerfulAre2019}. These results imply
  that if the local neighborhoods of two nodes $v, w$ in input graph $G$ are
  equal, then any \gnn will treat $v$ and $w$ identically. We may intuitively
  exploit this property for compression: if $v$ and $w$ are treated identically
  there is no need for them both to be present during learning; having one of
  them suffices. We fully develop this intuition in
  Section~\ref{sec:methodology}, where we also consider a more relaxed notion of
  ``local neighborhood''  that is applicable
  only to specific kinds of aggregate-combine \gnns.

  (3.) We empirically evaluate the effectiveness of our methodology in
  Section~\ref{sec:evaluation}. In particular, we give insights into the
  compression ratios that can be obtained on real-world graphs. While we find
  that these ratios are diverse, from compressing extremely well to compressing
  only marginally, a preliminary experiment on an existing \gnn benchmark shows positive impact on learning efficiency  even with marginal compression.

  We start with preliminaries in Section~\ref{sec:prelims} and conclude in
  Section~\ref{sec:conclusion}.  Proofs of formal statements may be found in the
  Appendix.

\section{Preliminaries}
\label{sec:prelims}
\paragraph{Background.} We denote by $\reals$ the set of real numbers, by $\nats$ the set of natural numbers, and by $\natsinf$ the set $\natnum \cup \{\infty\}$ of natural numbers extended with infinity.
We will use double curly braces $\mset{\dots}$ to denote multisets and multiset
comprehension. Formally, we view a multiset over a domain of elements $S$ as a
function $M \colon S \to \nats$ that associates a multiplicity $M(x)$ to each
element $x \in S$. As such, in the multiset
$M = \multileft a, a, b \multiright$, we have that $M(a) = 2$ and $M(b) = 1$.
If $M(x) = 0$ then $x$ is not present in $M$.  We 
denote by $\supp(M)$ the set of all elements present in $M$,
$\supp(M) := \{ x \in S \mid M(x) > 0 \}$. 
Note that if
every element has multiplicity at most one, then $M$ is a set. If $M$ is a multiset
and $c \in \natsinf$ then we denote by $\mrestr{M}{c}$ the multiset obtained
from $M$ by restricting the multiplicity of all elements to be at most $c$,
i.e., $\mrestr{M}{c}(x) = \min(M(x), c)$, for all elements $x$. Note in
particular that $\mrestr{M}{+\infty} = M$ and that $\mrestr{M}{1}$ converts $M$
into a set.

\paragraph{Graphs.} We work with directed node-colored multigraphs.  Formally,
our graphs are hence triples $G = (V, E, g)$ where $V$ is a finite set of nodes;
$E$ is a \emph{multiset} of edges over $V \times V$; and $g$ is a function,
called the \emph{coloring} of $G$, that maps every node $v \in V$ to a
\emph{color} $g(v)$. (The term ``color'' is just an intuitive way to specify
that $g$ has some unspecified range.)  If $Y$ is the co-domain of $g$, i.e., $g$
is of the form $g\colon V \to Y$ then we also call $g$ a \emph{$Y$-coloring} and
say that $G$ is a \emph{$Y$-colored graph}, or simply a $Y$-graph. When
$Y = \reals^n$ we also call $g$ an $n$-dimensional \emph{feature map}.  To ease
notation we write $v \in G$ to indicate that $v \in V$. Furthermore, we write
$G(v)$ instead of $g(v)$ to denote the color of $v$ in $G$, and we write
$G(v \to w)$ instead of $E(v \to w)$ to denote the multiplicity of edge
$v \to w$ in $G$. When $E$ is a set, i.e., when every edge has multiplicity at
most one, then we also call $G$ a \emph{simple} graph. We write 
$\inc_G(v)$ for the multiset $\mset{w \in G \mid w \to v \in E}$ of all incoming
neighbors of $v$. So, if the edge $w \to v$ has multiplicity $5$ in $E$ then $w$
also has multiplicity $5$ in $\inc_G(v)$.  We drop subscripts when the graph $G$
is clear from the context. The \emph{size} of a graph $G$ is the number of nodes
$\size{V}$ plus the number of simple edges $\size{\!\supp(E)}$. This is a
reasonable definition of the size of a multigraph, since for each edge it
suffices to simply store its multiplicity as a number, and storing a number
takes unit cost in the RAM model of computation.

\paragraph{Color transformers.} If $\ctrans$ is a function that maps $X$-colored
graphs $G = (V,E,g)$ into $Y$-colored graphs $G' = (V',E', g')$ that leaves
nodes and edges untouched and only changes the coloring, i.e., $V = V'$ and
$E = E'$ then we call $\ctrans$ a \emph{coloring transformer}. In particular, if
$X = \reals^p$ and $Y = \reals^q$ for some dimensions $p$ and $q$ then $\ctrans$ is
a \emph{feature map transformer}.

\paragraph{Graph Neural Networks.} Graph Neural Networks (\gnns) are a
popular form of neural networks that enable deep learning on graphs. Many
different forms of \gnns have been proposed in the literature. We refer the
reader to the overview by Hamilton
\cite{hamiltonGraphRepresentationLearning2020}. In this paper we focus on a
standard form of \gnns that is known under the name of \emph{aggregate-combine
  \gnns}~\cite{barceloExpressivePowerGraph2020}, also called
\emph{message-passing \gnns}. These are defined as
follows~\cite{groheLogicGraphNeural2021,geertsExpressivePowerMessagePassing2022}.

A \emph{\gnn layer} of input dimension $p$ and output dimension $q$ is a pair $(\agg, \comb)$ of functions  where (1) $\agg$ is an \emph{aggregation function} that maps finite multisets of vectors in $\reals^p$ to vectors in $\reals^h$ for some dimension $h$ and (2) $\comb$ is a \emph{combination function} $\comb\colon \reals^p \times \reals^h \to \reals^q$. In practice, $\agg$ is usually taken to compute the arithmetic mean, sum, or maximum of the vectors in the multiset, while $\comb$ is computed by means of a feedforward neural network whose parameters can be learned.

A \emph{\gnn} is a sequence $\L = (L_1,\dots, L_k)$ of \gnn layers, where the output dimension of $L_i$ equals the input dimension of $L_{i+1}$, for $1 \leq i < k$. The \emph{input and output dimensions} of the \gnn are the input dimension of $L_1$, and the output dimension of $L_k$ respectively. In what follows, we write $\L\colon p,q$ to denote that $p$ is the input dimension of $\L$ and $q$ is the output dimension.

Semantically, \gnn layers and \gnns are feature map transformers~\cite{geertsExpressivePowerMessagePassing2022} In particular, when \gnn layer $L = (\agg,\comb)$  of input dimension $p$ and output dimension $q$ is executed on $\reals^p$-colored graph $G = (V,E,g)$ it returns the $\reals^q$-colored graph $G'= (V,E, g')$ with $g'$ the $q$-dimensional feature map defined by
\[ g'\colon v \mapsto \comb \big( g(v), \agg\,\mset{ g(w) \mid w \in \inc_G(v)}
	\big). \] As such, for each node $v$, $\L$ aggregates the (multiplicity-weighted)
$\reals^p$ colors of $v$'s neighbors, and combines this with $v$'s own color to
compute the $\reals^q$ output.

A \gnn $\L\colon p,q$ simply composes the transformations defined by its layers: given $\reals^p$-colored graph $G$ it returns the $\reals^q$-colored graph $(L_k \circ L_{k-1} \circ \dots \circ L_1)(G)$.

\paragraph{Discussion.} It is important to stress that in the literature \gnns
are defined to operate on \emph{simple graphs}, whereas we have generalized
their semantics above to also work on \emph{multigraphs}. We did so because,
as we will see in Section~\ref{sec:methodology}, the result of compressing a
simple graph for the purpose of learning naturally yields a multigraph.

\paragraph{Learning problems.}
\gnns are used for a wide range of supervised learning tasks on graphs. For example, for a node $v$, the $\reals^q$-vector $\L(G)(v)$ computed for $v$ by \gnn $\L$ can be interpreted, after normalisation, as a probability distribution over $q$ new labels (for node classification),
or as predicted values (for node regression). Similarly, an edge prediction for nodes $v$ and $w$ can be
made based on the pair $(\L(G)(v), \L(G)(w))$. Finally, by aggregating $\L(G)(v)$ over all nodes $v \in G$, one obtains
graph embeddings that can be used for graph classification, regression and clustering~\cite{hamiltonGraphRepresentationLearning2020}.

In this work, we focus on the tasks of node classification and
regression.  Our methodology is equally applicable to the other
tasks, however.

In order to make precise what we mean by learning \gnns on compressed graphs for
node classification, we propose the following formal definition.
\begin{definition}
	A \emph{learning problem} of input dimension $p$ and output dimension $q$ is a tuple $\prob = (G, \train, \loss, \topol)$ where
	\begin{compactitem}
		\item $G$ is the $\reals^p$-colored graph on which we wish to learn;
		\item $\train$ is a subset of $G$'s nodes, representing the training set;
		\item $\loss\colon \train \times \reals^q \to \reals$ is a loss function that allows to quantify, for each node $v \in \train$ the dissimilarity $\loss(v,c)$ of the $\reals^q$-color $c$ that is predicted for $v$ by a \gnn and the desired $\reals^q$-color for $v$ as specified in the training set;
		\item $\topol$ is the \emph{hypothesis space}, a (possibly infinite)
		collection of \gnns of input dimension $p$ and output dimension $q$.
	\end{compactitem}
\end{definition}

Given a learning problem $\prob$, a learning algorithm produces a ``learned''
\gnn in $\topol$ by traversing the search space $\topol$. 
For each
currently considered \gnn $\L \in \topol$, the observed loss of $\L$ on $G$
w.r.t. $\train$ is computed as
\[ \loss( \L(G), \train ) := \sum_{v \in \train} \loss\big(v, \L(G)(v) \big). \] The learning algorithm aims to minimize this loss,
but possibly returns an $\L$ for which this is only a local minimum.

In practice, $\topol$ is usually a collection of \gnns with the same topology:
they all have the same number of layers (with each layer $d$ having  the same input and output dimensions accross \gnns in $\topol$) and are
parametrized by the same number of learnable parameters. 
Each concrete parametrization
constitutes a concrete \gnn in $\topol$ in our framework.  Commonly, the learned
\gnn $\L$ is then found by means of gradient descent, which updates the
learnable parameters of the \gnns in $\topol$ to minimize the observed
error.

No matter which concrete learning algorithm is used to solve a learning problem,
the intent is that the returned $\L$ generalizes well: it makes predictions on
$G$ for the nodes not in $\train$, and can also be applied to other,
new $\reals^p$-colored graphs to predict $\reals^q$-vectors for each node.

\paragraph{Our research question in this paper is the following.}

\begin{quote}
	\it Given a \gnn learning problem $\prob = (G, \train, \loss, \topol)$, is it
	possible to transform this into a new problem
	$\prob'= (G', \train', \loss', \topol)$ that is obtained by \emph{compressing}
	$G$, $\train$, and $\loss$ into a smaller graph $G'$, training
	set $\train'$, and
	loss function $\loss'$ such that instead of learning a \gnn on $\prob$ we
	could equivalently learn a \gnn on $\prob'$ instead?
\end{quote}
Here ``equivalently'' means that ideally, no matter which learning algorithm is
used, we would like the learned \gnn to be identical in both cases. Of course,
this is not possible in practice because the learning process is itself
non-deterministic, e.g., because the learning algorithm makes random starts;
because of stochasticity in stochastic gradient descent; or because of
non-deterministic dropout that is applied between layers. Nevertheless, we expect the \gnn obtained by learning on the compressed
problem would perform ``as good'' as the \gnn obtained by the learning on the
uncompressed problem, in the sense that it generalizes to unseen nodes and
unseen colored graphs equally well.

To ensure that we may hope any learning algorithm to perform equally well on $\prob'$ as on $\prob$, we formally define:

\begin{definition}
	\label{def:equiv-problems}
	Two learning problems $\prob$ and $\prob'$ are \emph{equivalent}, written
	$\prob \equiv \prob'$, if they share the same hypothesis space of \gnns
	$\topol$ and, for every $\L \in \topol$ we have
	$\loss( \L(G), \train ) = \loss'(\L(G'), \train')$.
\end{definition}
In other words, when traversing the hypothesis space for a \gnn to return, no learning algorithm can distinguish between $\prob$ and $\prob'$. All other things being equal, if the learning algorithm then returns a \gnn $\L$ when run on $\prob$, it will return $\L$ on $\prob'$ with the same probability.

Note that, while the hypothesis space $\topol$ remains unchanged in this
definition, it is possible (and, as we will see, actually required) to
adapt the loss function $\loss$ into a modified loss function $\loss'$ during
compression.

The benefit of a positive answer to our research question, if compression is
efficient, is computational efficiency: learning on smaller graphs is faster
than learning on larger graphs and requires less memory.

\section{Methodology}
\label{sec:methodology}
To compress one learning problem into an equivalent, hopefully smaller, problem
we will exploit recent insights into the expressiveness of
\gnns~\cite{morrisWeisfeilerLemanGo2019,xuHowPowerfulAre2019,barceloExpressivePowerGraph2020}. In
particular, it is known that if the local neighborhoods of two nodes $v, w$ in
input graph $G$ are equal, then any \gnn will treat $v$ and $w$ identically. In
particular, it will assign the same output colors to $v$ and $w$. We may
intuitively exploit this property for compression: since $v$ and $w$ are treated
identically there is no need for them both to be present during learning; having
one of them suffices. So, we could compress by removing nodes that are redundant in this sense. We must take
care, however, that by removing one, we do not change the structure (and hence,
possibly, the predicted color) of the remaining node. Also, of course, we need
to make sure that by removing nodes we do not lose training information. I.e.,
if $\train$ specifies a training color for $v$ but not $w$ then if we decide to
remove $v$, we somehow need to ``fix'' $\train$, as well as the loss function.

This section is devoted to developing this intuitive idea. In Section~\ref{sec:indist} we first study under which 
conditions  \gnns treat nodes identically. Next, in Section~\ref{sec:reducts} we develop compression of colored
graphs based on collapsing identically-treated nodes, allowing to remove  redundant nodes
while retaining the structure of the remaining nodes. 
Finally, in Section~\ref{sec:problem-compr}, we discuss compression of the training set and loss function. Together, these three ingredients allow us to compress a learning problem into an equivalent problem, cf.~Definition~\ref{def:equiv-problems}.

We close this section by proposing an alternative definition of compression that
works only for a limited class of learning problems.  It is nevertheless
interesting as it may allow better compression,  as we will show in Section~\ref{sec:evaluation}. 

\subsection{Indistinguishability}
\label{sec:indist}

The following definition formalizes when two nodes, not necessarily in the same graph, are treated identically by a class of \gnns.

\begin{definition}
  Let $\topol$ be a class of \gnns, let $G$ and $H$ be two
  $\reals^p$-colored graphs for some $p$, and let $v\in G,w \in H$ be two nodes
  in these graphs. We say that $(G,v)$ is $\topol$-\emph{indistinguishable}
  from $(H,w)$, denoted $(G,v) \indist{\topol} (H,w)$, if for every \gnn
  $\L \in \topol$ of input dimension $p$  it holds that
  $\L(G)(v) = \L(H)(w)$. 
\end{definition}
In other words, two nodes are indistinguishable by a class of \gnns $\topol$ if no $\gnn$ in $\topol$ can ever assign a different output color to these nodes, when started on $G$ respectively $H$. We call $(G,v)$ and $(H,w)$  $\topol$-\emph{distinguishable} otherwise. 

For the purpose of compression, we are in search of sets of nodes in the input
graph $G$ that are pairwise $\topol$-indistinguishable, with $\topol$ the
hypothesis space of the input learning problem. It is these nodes that we can
potentially collapse in the input learning problem. 
Formally, let $\class{G,v}{\topol}$ denote the set of all nodes in $G$ that are $S$-indistinguishable from $v$,
\[ \class{G,v}{\topol} := \{w \in G \mid (G,v) \indist{\topol} (G,w) \}.\]
We aim to calculate $\class{G,v}{\topol}$ and subsequently compress $G$ by removing all but one node in $\class{G,v}{\topol}$ from $G$.

\paragraph*{Color refinement.}
To calculate $\class{G,v}{\topol}$, we build on the work of Morris et
al.~\cite{morrisWeisfeilerLemanGo2019} and Xu et
al.~\cite{xuHowPowerfulAre2019}. They proved independently that a GNN can
distinguish two nodes if an only if the so-called \emph{color refinement}
algorithm assigns different colors to these nodes. Color refinement is
equivalent to the one-dimensional Weisfeiler-Leman (WL)
algorithm~\cite{groheLogicGraphNeural2021}, and works as follows.

\begin{definition}
  \label{def:color-refinement}
  The (one-step) \emph{color-refinement} of colored graph $G=(V,E,g)$, denoted
  $\colr(G)$, is the colored graph $G'= (V,E,g')$ where $g'$ maps every node
  $v \in G$ to a pair, consisting of $v$'s original color and the multiset of
  colors of its incoming neighbors:
\[ g'\colon v \mapsto \big( G(v), \mset{ G(w) \mid w \in
    \inc_G(v)}\big).\] As such, we can think of $\colr(G)(v)$ as representing
the immediate neighborhood of $v$ (including $v$), for any node $v$.
\end{definition}
We denote by $\colr^d(G)$  the result of applying $d$ color refinement steps on $G$, so $\colr^0(G) = G$ and $\colr^{d+1}(G) = \colr(\colr^d(G))$. Using this notation, we can think of  $\colr^d(G)(v)$ as representing the local neighborhood of $v$ ``up to radius $d$''.

To illustrate, Figure~\ref{fig:color-refinement} shows a colored graph $G$ and two steps of color refinement. 

\newcommand{\circref}[1]{\textnormal{\CircledText[outer color=blue]{\scriptsize{\ref{#1}}}}}

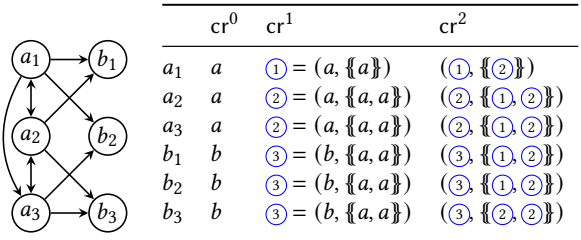
\begin{figure}
  \centering
  \begin{tikzpicture}[->,>=stealth, semithick, node distance=0.5cm]
    \tikzstyle{bnode}=[draw=black,text=black,inner sep=1pt, minimum size=5mm, circle]

    \node[bnode] (a1) {$a_1$};
    \node[bnode, below=of a1] (a2) {$a_2$};
    \node[bnode, below=of a2] (a3) {$a_3$};
    \node[bnode, right=of a1] (b1) {$b_1$};
    \node[bnode, below=of b1] (b2) {$b_2$};
    \node[bnode, below=of b2] (b3) {$b_3$};

    \draw[->] (a1) edge[bend right=30] (a3);
    \draw[<->] (a2) --  (a3);
    \draw[<->] (a2) --  (a1);
    \draw[<-] (b1) -- (a1);
    \draw[<-] (b1) -- (a2);
    \draw[<-] (b2) -- (a1);
    \draw[<-] (b2) -- (a3);
    \draw[<-] (b3) -- (a2);
    \draw[<-] (b3) -- (a3);

    \node[minimum width=6cm,anchor=west] at (1.5,-0.75) {
      \begin{tabular}[t]{@{}lllll@{}}
        \toprule
        &$\colr^0$&$\colr^1$&$\colr^2$ \\ \midrule
        $a_1$ & $a$ & \cstep\label{c1} $=(a,\mset{a})$                                                                 
& $(\circref{c1}, \mset{\circref{c2}})$
 \\
        $a_2$ & $a$ & \cstep\label{c2} $=(a,\mset{a,a})$ &  $(\circref{c2}, \mset{\circref{c1},\circref{c2}})$
\\        
        $a_3$ & $a$ & \circref{c2} $=(a,\mset{a,a})$ & 
 $(\circref{c2}, \mset{\circref{c1},\circref{c2}})$
\\
        $b_1$ & $b$ & \cstep\label{c3} $=(b,\mset{a,a})$ & 
 $(\circref{c3}, \mset{\circref{c1},\circref{c2}})$
\\
        $b_2$ & $b$ & \circref{c3} $=(b,\mset{a,a})$ & 
 $(\circref{c3}, \mset{\circref{c1},\circref{c2}})$                                         
\\
        $b_3$ & $b$ & \circref{c3} $=(b,\mset{a,a})$  
& $(\circref{c3}, \mset{\circref{c2},\circref{c2}})$
\\
\end{tabular}
      };
  \end{tikzpicture}
  
  \caption{Example of color refinement. Nodes $a_1,a_2, a_3$ have the same color $a$; nodes $b_1,b_2,b_3$ have the same color $b$. All edges have multiplicity $1$.}
  \label{fig:color-refinement}
\end{figure}

The following property was observed by Morris et
al.~\cite{morrisWeisfeilerLemanGo2019} and Xu et
al~\cite{xuHowPowerfulAre2019} for \gnns operating on \emph{simple} graphs. We here extend it to multigraphs. 

\begin{propositionrep}
  \label{prop:indistinguishable-d}
  Let $\L$ be a  \gnn composed of $d \in \nats$ layers, $d \geq 1$. If
  $\colr^d(G)(v) = \colr^d(H)(w)$ then $\L(G)(v) = \L(H)(w)$. As a
  consequence, if $\topol$ is a hypothesis space consisting of \gnns of at most $d$ layers and
  $\colr^d(G)(v) = \colr^d(H)(w)$ then $(G,v) \indist{\topol} (H,w)$.
\end{propositionrep}
\begin{proof}
  It suffices to show that, for all graphs $G,H$ and nodes $v \in G, w \in H$ it
  holds that, if $\colr(G)(v) = \colr(H)(w)$ then $L(G)(v) = L(H)(w)$ for any
  GNN layer $L$. Using this observation, the proposition then follows for \gnns
  by straightforward induction on the number of layers $d$ and the fact that if
  $\colr^d(G)(v) = \colr^d(H)(w)$, then also
  $\colr^{d'}(G)(v) = \colr^{d'}(H)(w)$ for all $d' \leq d$.\footnote{This
    latter fact is because $\colr^{d}(G)(v)$ is a pair, whose first component
    is $\colr^{d-1}(G)(v)$ (which is a pair, whose first component is
    $\colr^{d-2}(G)(v)$, and so on), and similarly for $\colr^{d}(H)(w)$.}

    So, assume that $\colr(G)(v) = \colr(H)(w)$. Since, by definition,
    \begin{align*}
      \colr(G)(v) & = \left(G(v), \mset{G(v') \mid v' \in \inc_G(v)} \right)\\
      \colr(H)(w) & = \left(H(w), \mset{H(w') \mid w' \in \inc_H(w)} \right)
    \end{align*}
    we may conclude that, $G(v) = H(w)$ and $\mset{G(v') \mid v' \in \inc_G(v)} = \mset{H(w') \mid w' \in \inc_H(w)}$.
    Let $L = (\agg, \comb)$ be an arbitrary \gnn layer. Then, 
    \begin{align*}
      \L(G)(v) & = \comb(G(v), \agg \mset{G(v') \mid v' \in \inc_G(v)}) \\
               & = \comb(H(w), \agg \mset{H(w') \mid w' \in \inc_H(w)}) \\
      & = \L(H)(w) \qedhere
    \end{align*}
\end{proof}
In other words, $d$-layer \gnns cannot distinguish nodes that are assigned the
same color by $d$ steps of color refinement. 

Let $\class{G,v}{d}$ denote the set of all nodes in $G$ that receive the same color as $v$ after $d$ steps of color refinement,
\[ \class{G,v}{d} := \{ w \in G \mid \colr^{d}(G)(v) = \colr^{d}(G)(w)
  \}.\] Then it follows from Proposition~\ref{prop:indistinguishable-d} that
$\class{G,v}{d}$ is a \emph{refinement} of $\class{G,v}{\topol}$ in the
sense that $\class{G,v}{d} \subseteq \class{G,v}{\topol}$, for all
$v \in G$. 
Morris et
al.~\cite{morrisWeisfeilerLemanGo2019} and Xu et al.~\cite{xuHowPowerfulAre2019} have also shown
that for every graph $G$ and every depth $d$ there exists a \gnn $\L$ of $d$ layers such that $\class{G,v}{d} = \class{G,v}{\{\L\}}$, for every node $v \in G$. Consequently, if, in addition to containing only \gnns with at most $d$ layers, $\topol$ includes
\emph{all possible} $d$-layer \gnns, then
$\class{G,v}{d} = \class{G,v}{\topol}$ coincide, for all $v \in G$. Hence,
for such $\topol$ we may calculate $\class{G,v}{\topol}$ by calculating
$\class{G,v}{d}$ instead. When $\topol$ does not include all $d$-layer \gnns
we simply use $\class{G,v}{d}$ as a proxy for $\class{G,v}{\topol}$. This is
certainly safe: since $\class{G,v}{d} \subseteq \class{G,v}{\topol}$ no
\gnn in $\topol$ will be able to distinguish the nodes in $\class{G,v}{d}$ and
we may hence collapse nodes in $\class{G,v}{d}$ for the purpose of
compression. In this case, however, we risk that $\class{G,v}{d}$ contains too
few nodes compared to $\class{G,v}{\topol}$, and therefore may not provide
enough opportunity for compression.  We will return to this issue in Section~\ref{sec:graded}.

What happens if there is no bound on the number of layers of \gnns in $\topol$? In that case we can still use color refinement to compute $\class{G,v}{\topol}$ as follows. 
It is known that after a finite number of color refinements steps we reach a
value $d$ such that for all nodes $v\in G$ we have $\class{G,v}{d} = \class{G,v}{d+1}$. The smallest value $d$ for which this holds is called the \emph{stable coloring number} of $G$, and we denote the colored graph obtained by this value of $d$ by $\colr^{\infty}(G)$ in what follows. Similarly we denote the equivalence classes at this value of $d$ by $\class{G,v}{\infty}$.
From Proposition~\ref{prop:indistinguishable-d} it readily follows:
\begin{corollary}
  \label{cor:indistinguishable-inf}
  For any class $\topol$ of GNNs,
  if $\colr^{\infty}(G)(v) = \colr^{\infty}(H)(w)$ then $(G,v) \indist{\topol} (H,w)$.
\end{corollary}

We note that it is very efficient to compute the set
$\{ \class{G,v}{d} \mid v \in G\}$ of all color refinement classes: this can be
done in time $\mathcal{O}((n+m)\log n)$ with $n$ the number of vertices and $m$
the number of edges of the input graph~\cite{cardonPartitioningGraphLog1982}.  

\begin{example}
  To illustrate, consider the colored graph from
  Figure~\ref{fig:color-refinement}, as well as the color refinement steps
  illustrated there. (Recall that nodes $a_1,a_2,a_3$ share the same color, as do $b_1,b_2,b_3$.)  Then after one step of color refinement we have
  \begin{align*}
    \small
    \class{G,a_1}{1} & = \{ a_1 \} \\
    \class{G,a_2}{1} & = \class{G,a_3}{1} = \{ a_2, a_3 \} \\
    \class{G,b_1}{1} & =  \class{G,b_2}{1} = \class{G,b_3}{1} = \{ b_1,b_2,b_3\},
  \end{align*}
  while  after two steps of color refinement we obtain the following color refinement classes:
  \begin{align*}
    \small
    \class{G,a_1}{2} & = \{ a_1 \} & \class{G,b_1}{2} & = \class{G,b_2}{2} = \{ b_1,b_2\} \\
    \class{G,a_2}{2} & = \class{G,a_3}{2} = \{ a_2,a_3 \} &   \class{G,b_3}{2} & = \{ b_3\}.
  \end{align*}
  We invite the reader to check that for every node $v$ in this graph, $\class{G,v}{3} = \class{G,v}{2}$. As such, the stable coloring is obtained when $d=2$ and $\class{G,v}{2} = \class{G,v}{\infty}$.
\end{example}

\subsection{Graph reduction}
\label{sec:reducts}

Having established a way to compute redundant nodes, we now turn our attention
to compression. Assume that we have already computed the color refinement
classes $\{ \class{G,v}{d} \mid v \in G \}$ for $d \in
\natsinf$. For each $v \in G$, we wish to ``collapse'' all nodes in
$\class{G,v}{d}$ into a single node.  To that end, define a
\emph{$d$-substitution} on a graph $G$ to be a function that maps
each color refinement class in $\{ \class{G,v}{d} \mid v \in G \}$ to a node
$\rho(\class{G,v}{d}) \in \class{G,v}{d}$. Intuitively, $\rho(\class{G,v}{d})$ is the node that we wish to keep; all other nodes in $\class{G,v}{d}$ will be removed. In what follows we extend $\rho$ to also operate on nodes in $G$ by
setting $\rho(v) = \rho(\class{G,v}{d})$. 

\begin{definition}
  The \emph{reduction of graph $G$ by $d$-substitution $\rho$ on $G$} is the graph $H = (V, E, h)$ where
  \begin{itemize}
  \item $V = \{ \rho(v) \mid v \in G \}$
  \item For all $v, w \in V$ we have
    \[ E( v \to w ) = \sum_{v' \in \class{G,v}{d}} G(v' \to w) \]
    In particular, if there is no edge into $w$ in $G$, there will be no edge into $w$ in $H$, as this sum is then zero.
  \item nodes retain colors: for each node $v \in H$ we have $h(v) = G(v)$.
  \end{itemize}
  In what follows, we denote by $G/\rho$ the reduction of $G$ by $\rho$. A graph
  obtained by reducing $G$ according to some $d$-substitution $\rho$ is called a
  \emph{$d$-reduct of $G$}.
\end{definition}

\begin{example}
  \label{ex:reductions}
  Reconsider the colored graph $G$ of
  Figure~\ref{fig:color-refinement}. Let $\rho_1$ and $\rho_2$ be the following $d=1$-substitutions:
  \begin{align*}
    \rho_1\colon \{a_1\} & \mapsto a_1 & \{a_2,a_3\} & \mapsto a_2 & 
                                                                     \{b_1,b_2,b_3\} & \mapsto b_1 \\
    \rho_2\colon \{a_1\} & \mapsto a_1 & \{a_2,a_3\} & \mapsto a_2 & 
                                                                     \{b_1,b_2,b_3\} & \mapsto b_3
  \end{align*}
  Then $G/\rho_1$ and $G/\rho_2$ are illustrated in Figure~\ref{fig:reductions}.
\end{example}

\begin{figure}
  \centering
  \begin{tikzpicture}[->,>=stealth, semithick, node distance=0.5cm]
    \tikzstyle{bnode}=[draw=black,text=black,inner sep=1pt, minimum size=5mm, circle]
    \tikzstyle{wnode}=[text=gray,inner sep=1pt]

    \begin{scope}
    \node[bnode] (a1) at (0.0cm, 0.0cm) {$a_1$};
    \node[bnode] (a2) at (0.0cm,-1.0cm) {$a_2$};
    \node[bnode] (b1) at (1.0cm, 0.0cm) {$b_1$};

    \path[->] (a1) edge node[left,wnode] {$1$} (a2);
    \path[->] (a1) edge node[above,wnode] {$1$} (b1);
    \path[->] (a2) edge [loop below] node[wnode] {$1$} (a2);
    \path[->] (a2) edge node[below,wnode] {$1$} (b1);
    \end{scope}

    \begin{scope}[xshift=3cm]
    \node[bnode] (a1) at (0.0cm, 0.0cm) {$a_1$};
    \node[bnode] (a2) at (0.0cm,-1.0cm) {$a_2$};
    \node[bnode] (b3) at (1.0cm,-1.0cm) {$b_3$};

    \path[->] (a1) edge node[left,wnode] {$1$} (a2);
    \path[->] (a2) edge [loop below] node[wnode] {$1$} (a2);
    \path[->] (a2) edge node[below,wnode] {$2$} (b3);
    \end{scope}
  \end{tikzpicture}
  
  \caption{Reduction of the graph of Figure~\ref{fig:color-refinement} by the $1$-substitutions $\rho_1$ and $\rho_2$ from Example~\ref{ex:reductions}.}     \label{fig:reductions}
\end{figure}
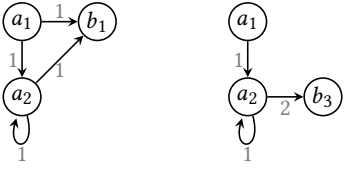

The following proposition is an essential property of our compression methodology.

\begin{propositionrep}
  \label{prop:compression-same-colr}
  For every graph $G$, every $d$-substitution $\rho$ of $G$ with $d \in \natsinf$, and every node $v \in G$ we have $\colr^d(G)(v)  = \colr^d(G/\rho)(\rho(v))$. 
\end{propositionrep}
\begin{proof}
  Fix $d \in \natsinf$ arbitrarily. When $d \not = \infty$, the statement follows from Proposition~\ref{prop:compression-same-colr-bounded-rounds} below. 
It hence remains to prove the case where $d = \infty$. To that end, let $c_1 \in \nats$ be the stable coloring number of $G$, and let $c_2\in \nats$ be the stable coloring number of $G/\rho$. We prove that $c_1 = c_2$. From this, the claimed equality directly follows: 
  \begin{align*}
    \colr^{\infty}(G)(v)  & = \colr^{c_1}(G)(v) \\
                          & = \colr^{c_2}(G)(v) \\
                          & = \colr^{c_2}(G/\rho)(\rho(v)) \\
    & = \colr^{\infty}(G/\rho)(\rho(v))
  \end{align*}
  Here, the first and last equality are by definition of $\colr^{\infty}$; the
  second equality is because $c_1 = c_2$; and the third equality is by
  Proposition~\ref{prop:compression-same-colr-bounded-rounds}.

  To prove that $c_1 = c_2$ it suffices to show that  $c_1 \geq c_2$ and $c_1 \leq c_2$. We only explain how to obtain the inequality $c_1 \geq c_2$; the other direction is similar.

  To show that  $c_1 \geq c_2$, we must show that $\class{G/\rho,\rho(v)}{c_1} = \class{G/\rho,\rho(v)}{c_1+1}$ for all nodes $\rho(v) \in G/\rho$. Since  $\class{G/\rho,\rho(v)}{c_1} \supseteq \class{G/\rho,\rho(v)}{c_1+1}$ trivially holds by definition of $\colr$, it remains to show that if $\rho(w) \in \class{G/\rho,\rho(v)}{c_1}$ then also $\rho(w) \in \class{G/\rho,\rho(v)}{c_1+1}$, for all nodes $\rho(w) \in G/\rho$. Hence, fix $\rho(w) \in G/\rho$ and  assume $\colr^{c_1}(G/\rho)(\rho(v)) = \colr^{c_1}(G/\rho)(\rho(w))$. Then, by  Proposition~\ref{prop:compression-same-colr-bounded-rounds} we derive
  \begin{align*}
    \colr^{c_1}(G)(v) & = \colr^{c_1}(G/\rho)(\rho(v)) 
                      =  \colr^{c_1}(G/\rho)(\rho(w))
                      = \colr^{c_1}(G)(w)
  \end{align*}
  Since $c_1$ is the stable coloring number of $G$ and
  $\colr^{c_1}(G)(v) = \colr^{c_1}(G)(w)$,  also
  $\colr^{c_1+1}(G)(v) = \colr^{c_1+1}(G)(w)$. Therefore, again by Proposition
  Proposition~\ref{prop:compression-same-colr-bounded-rounds},
  \begin{align*}
    \colr^{c_1+1}(G/\rho)(\rho(v))  & = \colr^{c_1+1}(G)(v)  \\
                     & =  \colr^{c_1+1}(G)(w)\\
                      & = \colr^{c_1+1}(G/\rho)(\rho(w)).
  \end{align*}
  As such, $\rho(w) \in \class{G/\rho, \rho(v)}{c_1+1}$, as desired.
\end{proof}

\begin{toappendix}

  \begin{proposition}
    \label{prop:compression-same-colr-bounded-rounds}
    For every graph $G$, every $d$-substitution $\rho$ of $G$ with $d \in \nats$
    (i.e., $d \not = \infty$), and every node $v \in G$ we have
    $\colr^d(G)(v) = \colr^d(G/\rho)(\rho(v))$.
  \end{proposition}
  \begin{proof}
    Let $d$ be an arbitrary but fixed natural number, and let $\rho$ be a
    $d$-substitution of $G$. For ease of
    readability, let us abbreviate $G/\rho$ by $H$.

    Throughout the proof, we will use the following  observations.

    (O1.) By definition of $d$-substitutions,
    $\rho(\class{G,v}{d}) \in \class{G,v}{d}$ for every node $v \in G$. Because we
    have extended $d$-substitutions to nodes by setting
    $\rho(v) = \rho(\class{G,v}{d})$, it follows in particular that
    $\rho(v) \in \class{G,v}{d}$ for every node $v$. Therefore, by definition
    of $\class{G,v}{d}$ we have $\colr^d(G)(v) = \colr^d(G)(\rho(v))$ for every
    $v \in G$. Because $d$-step color refinement includes $d-1$-step color
    refinement in the first component of the pair that it outputs, this also
    implies that $\colr^l(G)(v) = \colr^l(G)(\rho(v))$ for every $l \leq d$.

    (O2.) In addition, we note that for every node $w \in H = G/\rho$ we have
    $\rho(w) = w$. Indeed: only nodes that appear in the image of $\rho$ are in
    $G/\rho$. As such, if $w \in G/\rho$ there exists some $w'$ such that
    $w = \rho(\class{G,w'}{d})$. Because
    $\rho(\class{G,w'}{d}) \in \class{G,w'}{d}$, it follows that
    $w \in \class{G,w'}{d}$ and as such, $\class{G,w}{d} =
    \class{G,w'}{d}$. Hence, $w = \rho(\class{G,w}{d}) = \rho(w)$.

    (O3.) In addition, we note that for every node $w \in H$ and every node
    $w'\in \class{G,w}{d}$ we have $\rho(w') = w$. Indeed, because
    $w'\in \class{G,w}{d}$ we have $\class{G,w'}{d} = \class{G,w}{d}$. As such
    $\rho(w') = \rho(\class{G,w'}{d}) = \rho(\class{G,w}{d}) = \rho(w) = w$, where
    the last equality is due to observation (O2). 

    (O4.) For every node $v \in H$ we have $\inc_H(v) = \mset{\rho(w) \mid w \in \inc_G(v)}$. Indeed, let $E_H$ denote the multiset of edges of $H$ and $E_G$ the multiset of edges of $G$. Then
    \begin{align*}
      \inc_H(v) & = \mset{ w \mid (w \to v) \in E_H} \\
                & = \mset{w \mid w' \in \class{G,w}{d}, (w'\to v) \in E_G} \\
                & = \mset{w \mid \rho(w') = w, (w'\to v) \in E_G} \\
                & = \mset{\rho(w') \mid (w' \to v) \in E_G} \\
                & = \mset{\rho(w') \mid w' \in \inc_G(v) } \\
                & = \mset{\rho(w) \mid w \in \inc_G(v) }
    \end{align*}
    The first equality is by definition of $\inc_H$; the second because by definition the multiplicity of $w \to v$ in $E_H$ equals $\sum_{w' \in \class{G,w}{d}} E_G(w'\to v)$; and  the third by (O3).

    To prove the proposition, we now prove the stronger statement that for every
    $l \leq d$ and every node $v \in G$ we have
    $\colr^l(G)(v) = \colr^l(G/\rho)(\rho(v))$. Clearly
    $\colr^d(G)(v) = \colr^d(G/\rho)(\rho(v))$ then follows when $l = d$.

    The proof of this stronger statement is by induction on $l$. 
    For the base case when $l = 0$ we have, for every $v \in G$,
    \[ \colr^0(H)(\rho(v)) = H(\rho(v)) = G(\rho(v)) = \colr^0(G)(\rho(v)) =
      \colr^0(G)(v)\] as desired. Here, the first equality is by definition of
    $\colr^0$; the second  is by definition of $H = G/\rho$; the third is
    again by definition of $\colr^0$; and the final equality is by observation (O1).

    For the inductive case $l>0$, consider an arbitrary node $v \in G$.  Then, by
    definition,
    \begin{align*}
      \colr^l(G)(v) & = (\colr^{l-1}(G)(v), \mset{ \colr^{l-1}(G)(w)  \mid w \in \inc_G(v) }) \\
      \colr^l(H)(\rho(v))    & = (\colr^{l-1}(H)(\rho(v)), \mset{ \colr^{l-1}(H)(w')  \mid w' \in \inc_{H}(\rho(v)) }) 
    \end{align*}
    To show that $\colr^l(G)(v) = \colr^l(H)(\rho(v))$ we hence need to show:
    \begin{compactenum}[(i)]
    \item $\colr^{l-1}(G)(v) = \colr^{l-1}(H)(\rho(v))$; and
    \item
      $\mset{ \colr^{l-1}(G)(w) \mid w \in \inc_G(v) }) = \mset{
        \colr^{l-1}(H)(w') \mid w' \in \inc_{H}(\rho(v))}$.
    \end{compactenum}
    Equality (i) follows directly from the induction hypothesis. To show (ii), we reason as follows:
    \begin{align*}
      \mset{ & \colr^{l-1}(H)(w')  \mid w'  \in \inc_H(\rho(v)) }) \\
             & =     \mset{ \colr^{l-1}(H)(w')  \mid w' \in \mset{ \rho(w) \mid w \in \inc_G(\rho(v)) }}) \\
             & =     \mset{ \colr^{l-1}(H)(\rho(w))  \mid w \in \inc_G(\rho(v)) } \\
             & =     \mset{ \colr^{l-1}(G)(w)  \mid w \in \inc_G(\rho(v)) } \\
             & =   \mset{ \colr^{l-1}(G)(w)  \mid w \in \inc_G(v) }
    \end{align*}
    The first equality is due to observation (O3); the second is an elementary equality of multiset comprehensions;  the third is by induction hypothesis; and the fourth is because $\colr^l(G)(v) = \colr^l(G)(\rho(v)$ by Observation (O1). In particular, the next-to-last line is exactly the second component of the color $\colr^l(G)(\rho(v)$ and the last line  is the second component of $\colr^l(G)(v)$, which must hence be equal.
  \end{proof}
\end{toappendix}

It hence follows from Proposition~\ref{prop:indistinguishable-d} and
Corollary~\ref{cor:indistinguishable-inf} that if $\topol$ consists of \gnns of
at most $d \in \natsinf$ layers, then $(G,v) \indist{\topol} (G/\rho, \rho(v))$.

\paragraph*{Discussion} Example~\ref{ex:reductions} shows that the choice of $d$-substitution determines the reduct $G/\rho$ that we obtain. In particular, the example shows that distinct substitutions can yield distinct, \emph{non-isomorphic} reducts.  This behavior is unavoidable, unless $d$ is the stable coloring number of $G$. Indeed, we are able to show:

\begin{toappendix}
  \begin{lemma}
    \label{lem:isom-agrees-with-substitution}
    Let $H_1 = G/\rho_1$ and $H_2 = G/\rho_2$ be two $d$-reducts of $G$ and let $f$ be an isomorphism from $H_1$ to $H_2$. Then $f$ agrees with $\rho_2$: $f(u) = \rho_2(v)$ for all $u \in H_1$.
  \end{lemma}
  \begin{proof}
    It is straightforward to verify by induction on $d$ that $\colr^d$ is invariant under isomorphism, in the sense that $\colr^d(H_1)(u) = \colr^d(H_2)(f(u))$ for all $u \in H_1$. Therefore, for all $u \in H_1$
\begin{align*}
  \colr^d(G)(u) & = \colr^d(H_1)(\rho_1(u)) \\
                & = \colr^d(H_1)(u) \\
                & = \colr^d(H_2)(f(u)) \\
                & = \colr^d(H_2)(\rho_2(f(u))) \\
                & = \colr^d(G)(f(u))
\end{align*}
The first equality is by Proposition~\ref{prop:compression-same-colr}; the second by the fact that $\rho_1(u) = u$ for all $u \in H_1$%
; the third by invariance under isomorphisms; the fourth by the fact that $f(u) \in H_2$ and $\rho_2(u') = u'$ for every $u' \in H_2$
; and the last again by Proposition~\ref{prop:compression-same-colr}.

In other words, $f(u) \in \class{G,u}{d}$, for every $u \in H_1$. Then, because $f(u) \in H_2$ and $H_2$ contains only one node for each color refinement class in $\{ \class{G,u}{d} \mid u \in G\}$, it follows that $\rho_2(u) = f(u)$, for every $u \in H_1$ (including $v$).
  \end{proof}
\end{toappendix}

\begin{propositionrep}
  \label{prop:non-isomorphic-reducts}
  There is a single $d$-reduct of a graph $G$
  up to isomorphism if and only if $d$
  is greater than or equal to the stable coloring number of $G$.
\end{propositionrep}
\begin{proof}
\underline{(If.)}  For the if-implication, assume that $c$ is the stable coloring number of $G$ and let $d \geq c$, $d \in \nats$. Then, for all nodes $v \in G$ we have $\class{G,v}{d} = \class{G,v}{d+1}$. Let $H_1 = G/\rho_1$ and $H_2 = G/\rho_2$ be two $d$-reducts of $G$. Let $V_1 = \{ \rho_1(v) \mid v \in G\}$ be the set of nodes in $H_1$ and $V_2 = \{\rho_2(v) \mid v \in G\}$ be the set of nodes in $H_2$. It is clear that $V_1$ and $V_2$ are of the same cardinality, as they have one node for each color refinement class in $\{ \class{G,v}{d} \mid v \in G\}$.  It is furthermore straightforward to obtain that  $\rho_1(v) = v$ for all $v \in V_1$, and similarly $\rho_2(v) = v$ for all $v \in V_2$.

We claim that the function $f = \restr{\rho_2}{V_1}$ is an isomorphism from $H_1$ to $H_2$.

(1) This function is injective: assume that $v,w \in V_1$ and $\rho_2(v) = \rho_2(w)$. Since, by defintion of $d$-substitutions, we have $\rho_2(v) \in  \class{G,v}{d}$ and $\rho_2(w) \in \class{G,w}{d}$ it follows that  $\class{G,v}{d} = \class{G,w}{d}$. Because $v,w \in V_1$ we have $\rho_1(v) = v$ and $\rho_1(w) = w$. As such
  $v = \rho_1(v) = \rho_1(\class{G,v}{d}) = \rho_1(\class{G,w}{d}) = w$
  as desired.

(2) Since $f$ is an injective function from $V_1$ to $V_2$ and $V_1$ and $V_2$ are finite sets of the same cardinality, it is also surjective. Hence $f$ is a bijection between $V_1$ and $V_2$.

(3) It remains to show that for all $v,w \in V_1$ we have
  \[ H_1(v \to w) = H_2\left( \rho_2(v) \to \rho_2(w) \right).\]
Fix $v, w \in V_1$ arbitrarily. By definition of $\rho_2$ we have $\rho_2(w) \in \class{G,w}{d}$ and because $\class{G,w}{d} = \class{G,w}{d+1}$ ($d$ is larger than the stable coloring number of $G$) it follows that $\rho_2(w) \in \class{G,w}{d+1}$. Therefore,  $\colr^{d+1}(G)(w) = \colr^{d+1}(G)(\rho_2(w))$. In particular,
\[ \mset{ \colr^d(G)(v') \mid v' \in \inc_G(w) } = \mset{ \colr^d(G)(v') \mid v' \in
    \inc_G(\rho_2(w)) } \]
Hence, for any $v \in G$ we also have
\begin{multline}
  \tag{$\star$}
  \mset{ \colr^d(G)(v') \mid v' \in \inc_G(w), v' \in \class{G,v}{d} } \\= \mset{ \colr^d(G)(v') \mid v' \in
    \inc_G(\rho_2(w)), v' \in \class{G,v}{d}}
\end{multline}
Let us denote the total multplicity of a finite multiset $M$ by $\#M$, i.e., $\#M = \sum_{x} M(x)$.
Then we reason as follows.
\begin{align*}
  H_1(v \to w) & = \sum_{v' \in \class{G,v}{d}} G(v' \to w) \\
  & = \# \mset{v' \mid v' \in \inc_G(w), v' \in \class{G,v}{d} } \\
  & = \# \mset{\colr^{d}(G,v') \mid v' \in \inc_G(w), v' \in \class{G,v}{d} }\\
   & =\# \mset{\colr^{d}(G,v') \mid v' \in \inc_G(\rho_2(w)), v' \in \class{G,v}{d}  }  \\
   & = \# \mset{\colr^{d}(G,v') \mid v' \in \inc_G(\rho_2(w)), v' \in \class{G,\rho_2(v)}{d}  } | \\
   & = \# \mset{v' \mid v' \in \inc_G(\rho_2(w)), v' \in \class{G,\rho_2(v)}{d}}  | \\
               & = \sum_{v' \in \class{G,\rho_2(v)}{d}} G(v' \to \rho_2(w)) \\
  & = H_2(\rho_2(v) \to \rho_2(w))
\end{align*}
The first equality is by definition of $H_1$; the second is by rewriting the sum in multiset notation; the third because we are only interested in the total multiplicity of the multiset and not its elements; the fourth by $(\star$); the fifth because $\class{G,v}{d} = \class{G,\rho(v)}{d}$ by definition of $d$-substitutions; the sixth again because we care only about multiplicity and not the actual elements; the seventh by rewriting the multiset notation into a sum; and the last by definition of $H_2$.

\medskip
\noindent \underline{(Only if).} Assume that all $d$-reducts of $G$ are isomorphic. We need to show that for all $v \in G$ we have $\class{G,v}{d} = \class{G,v}{d+1}$. The containment $\class{G,v}{d} \supseteq \class{G,v}{d+1}$ trivially holds by definition of $\colr$. For the other containment, assume that $w \in \class{G,v}{d}$, i.e.,  $\colr^d(G)(v) = \colr^d(G)(w)$. We will show that also $\colr^{d+1}(G)(v) = \colr^{d+1}(G)(w)$, which is equivalent to saying that   $w \in \class{G,v}{d+1}$.

Recall that $\colr^{d+1}(G)(v)$ and $\colr^{d+1}(G)(w)$ are pairs by definition.
Since $\colr^d(G)(v) = \colr^d(G)(w)$ by assumption, the first components of of
these pairs are certainly equal. To prove
$\colr^{d+1}(G)(v) = \colr^{d+1}(G)(w)$, we hence only need to show that also
their second components are equal, i.e., that
\begin{equation}
  \label{eq:isom-p2}
  \tag{$\ddag$}
  \underbrace{\mset{ \colr^d(G)(u) \mid u \in \inc_G(v) }}_{=: M_1} = \underbrace{ \mset{ \colr^d(G)(u) \mid u \in
    \inc_G(w) }}_{=: M_2}
\end{equation}
To obtain this equality, consider two $d$-substitutions
$\rho_1$ and $\rho_2$ s.t. 
\begin{align*}
  \rho_1&\colon \class{G,v}{d} \mapsto v & 
  \rho_2&\colon \class{G,v}{d} \mapsto w 
\end{align*}
 In other words, $\rho_1(v)=\rho_1(w) = v$ and $\rho_2(v) = \rho_2(w) = w$. By assumption, $H_1 = G/\rho_1$ and $H_2 = G/\rho_2$ are isomorphic. Let $f$ be an isomorphism from $H_1$ to $H_2$. By Lemma~\ref{lem:isom-agrees-with-substitution} $f$ agrees with $\rho_2$: $f(u) = \rho_2(u)$ for all $u \in H_1$ (including $v$). 
Furthermore, by definition of isomorphism, for all $u \in H_1$ we have $H_1(u \to v) = H_2(f(u) \to f(v)) = H_2(\rho_2(u) \to \rho_2(v)) = H_2(\rho_2(u) \to w)$. 

To show that $M_1 = M_2$ we show that for every $x$, $M_1(x) \leq M_2(x)$ and $M_2(x) \leq M_1(x)$. We only illustrate the reasoning for $M_1(x) \leq M_2(x)$, the converse direction is similar. Consider an element $x$ in $M_1$, and let $m = M(x)$ be its multiplicity. Then there exists some $u \in \inc_G(v)$ such that $x = \colr^d(G)(u)$, and $x$ occurs as many times in $M_1$ as there are elements in $M'_1 := \mset{u' \mid u' \in \inc_G(v), u' \in \class{G,u}{d}}$, i.e., $m_1 = \#M'_1$.  Because $\class{G,u}{d} = \class{G,\rho_1(u)}{d}$ for all $u \in G$, it follows that
\begin{align*}
  m_1 & = \#\mset{u' \mid u' \in \inc_G(v), u' \in \class{G,u}{d}} \\
      & = \#\mset{u' \mid u' \in \inc_G(v), u' \in \class{G,\rho_1(u)}{d}} \\
      & = H_1(\rho_1(u) \to v) \\
      & = H_2(\rho_2(\rho_1(u)) \to w) \\
      & = \#\mset{u' \mid u' \in \inc_G(w), u' \in \class{G,\rho_2(\rho_1(u))}{d}} \\
  & = \#\mset{u' \mid u' \in \inc_G(w), u' \in \class{G,u}{d}} 
\end{align*}
The last equality is because $\class{G, \rho_2(\rho_1(u))}{d} = \class{G,u}{d}$ by definition of $d$-reducts. Notice that all elements $u'$ in the multiset on the last line will create one copy of $\colr^d(G)(u') = \colr^d(G)(u)$ in $M_2$.  As such, $x = \colr^d(G)(u)$ occurs at least $m$ times in $M_2$, as desired.
\end{proof}

A direct consequence of having  different non-isomorphic reducts when $d$ is less than the coloring number is that some of these reducts may be  smaller than others. In Example~\ref{ex:reductions}, $G/\rho_1$ has $3$ nodes and $4$ edges  while $G/\rho_2$ has $3$ nodes and only $3$ edges.  We may always obtain  a $d$-reduct of minimal size as follows. For $d \in \natsinf$, define the \emph{$d$-incidence} of a node $w \in G$, denoted $\incidence^d_G(w)$, to be the number of color refinement classes in $\{ \class{G,v}{d} \mid v \in G\}$ that contain an incoming neighbor of $w$. That is, the $d$-incidence of $w$ is the number of distinct classes in $\{ \class{G,u}{d} \mid u \in \inc_G(w) \}$.
The following proposition shows that we obtain a $d$-reduct of minimal size by 
 by choosing a $d$-substitution that maps color refinement classes to nodes of minimal $d$-incidence.

\begin{propositionrep}
  \label{prop:minimal-reduct}
  Let $G$ be a graph, let $d \in \natsinf$ and let $\rho$ be a $d$-substitution such that 
\[ \incidence^d_G(\rho(v)) = \min_{v' \in \class{G,v}{d}} \incidence^d_G(v') ,\] for every node $v \in G$. Then the size of $G/\rho$ is minimal among all $d$-reducts of $G$.
\end{propositionrep}
\begin{proof}
  Since $\infty$-substitutions are simply $d$-substitutions with $d \in \nats$ the stable coloring number of $G$, it suffices to show the statement for all $d \in \nats$.

  Fix $d \in \nats$,  let $C = \{ \class{G,v}{d} \mid v \in G\}$ be the color refinement classes of $G$ of depth $d$, and let $\rho$ be a $d$-substitution such that \[ \incidence^d_G(\rho(c)) = \min_{v' \in c} \incidence^d_G(v) ,\] for every refinement class $c \in C$. Let $H = G/\rho$. Furthermore, let $\mu$ be another $d$-substitution and let $U = G/\mu$. We will show that  $H$, viewed as a simple graph by ignoring edge multiplicities, has no more edges than $U$.

Note that for each class $c \in C$ there is exactly one corresponding node in $H$ (namely, $\rho(c)$) and one corresponding node in $U$ (namely $\mu(c)$). We claim that, for every $c \in C$, the indegree\footnote{That is, the number of nodes $w$ in $H$ having an outgoing edge to $\rho(c)$.} of $\rho(c)$  in $H$ is at most the indegree of $\mu(c)$ in $U$. Since the total number of simple edges in $H$ equals $\sum_{c \in C} \textsf{indegree}_H(\rho(c))$ and the total number of simple edges in $U$ similarly equals $\sum_{c \in C}\textsf{indegree}_U(\mu(c))$, it follows that $H$, viewed as a simple graph, has no more edges than $U$.

  To prove the claim, let $c \in C$. There is an edge from $\rho(c') \to \rho(c)$ in $H$ if and only if $\class{G,\rho(c')}{d} \cap \inc_G(\rho(c))$ is non-empty. Hence, since $\class{G,\rho(c')}{d} = c'$, the indegree of $\rho(c)$ in $H$ is exactly $\incidence^d_G(\rho(c))$. Similar reasoning shows that the indegree of $\mu(c)$ in $U$ equals $\incidence^d_G(\mu(c))$. As such,
  \begin{align*}
    \textsf{indegree}_H(\rho(c)) & = \incidence^d_G(\rho(c)) \\
                                 & = \min_{v' \in c} \incidence^d_G(v') \\
                                 & \leq \incidence^d_G(\mu(c)) \\
    & = \textsf{indegree}_U(\mu(c)) 
  \end{align*}
  The inequality in the third line is because $\mu(c) \in c$ by definition of $d$-substitution.
\end{proof}
When we report the size of $d$-reducts in our experiments
(Section~\ref{sec:evaluation}), we always report the minimal size among all
$d$-reducts.

\paragraph*{Discussion} Example~\ref{ex:reductions} illustrates that, depending
on the substitution used, the result of reducing a simple graph may be a
multigraph. In the literature, however, \gnns and color refinement are normally
defined to operate on simple graphs.  One may therefore wonder whether it is
possible to define a different notion of reduction that always yields a simple
graph when executed on simple, instead of a multigraph as we propose here. To
answer this question, consider the tree-shaped graph $G$ in
Figure~\ref{fig:multigraph-reduction}(left), whose root $v$ has $m$ $b$-colored
children, each having $n$ $c$-labeled children. It is straightforward to see
that when $d\geq 2$, any simple graph $H$ that has a node $w$ such that
$\colr^d(G)(v) = \colr^d(H)(w)$ must be such that $w$ has $m$ $b$-colored
children in $H$, each having $n$ $c$-labeled children.  As such, because $H$ is
simple, it must be of size at least as large as $G$. Instead, by moving to
multigraphs we are able to compress this ``regular'' structure in $G$ by only
three nodes, as show in Figure~\ref{fig:multigraph-reduction}(right). This
illustrates that for the purpose of compression we naturally need to move to
multigraphs.

\begin{figure}
  \begin{center}
    \small
  \begin{tikzpicture}[->,>=stealth, semithick, node distance=0.5cm]
    \tikzstyle{bnode}=[draw=black,text=black,inner sep=1pt, minimum size=5mm, circle]
    \tikzstyle{wnode}=[text=gray,inner sep=2pt]

    \node[bnode] (v) at (0.0cm, 0.0cm) {$v$};
    \node[bnode] (b_1) at (-1cm,-1cm) {$b_1$};
    \node (bdots) at (0cm,-1cm) {\dots};
    \node[bnode] (b_m) at (1cm,-1cm) {$b_m$};

    \node[bnode] (c_11) at (-1.5cm,-2cm) {$c^1_1$};
    \node  at (-1.0cm,-2cm) {\dots};
    \node[bnode] (c_1n) at (-0.5cm,-2cm) {$c^1_n$};

    \node[bnode] (c_m1) at (0.5cm,-2cm) {$c^m_1$};
    \node  at (1.0cm,-2cm) {\dots};
    \node[bnode] (c_mn) at (1.5cm,-2cm) {$c^m_n$};

    \path[->] (c_m1) edge (b_m);
    \path[->] (c_mn) edge (b_m);
    \path[->] (c_11) edge (b_1);
    \path[->] (c_1n) edge (b_1);
    \path[->] (b_1) edge (v);
    \path[->] (b_m) edge (v);

    \node[bnode] (v) at (3.0cm, 0.0cm) {$v$};
    \node[bnode] (b) at (3.0cm, -1.0cm) {$b_1$};
    \node[bnode] (c) at (3.0cm, -2.0cm) {$c^1_1$};

    \path[->] (c) edge  node[left,wnode] {$n$} (b);
    \path[->] (b) edge  node[left,wnode] {$m$} (v);

  \end{tikzpicture}
  
  \caption{Reduction of a tree-shaped simple graph (left) into a small multigraph (right).}
  \label{fig:multigraph-reduction}
\end{center}
\end{figure}
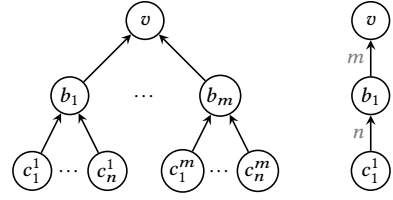

\subsection{Problem compression}
\label{sec:problem-compr}

We next combine the insights from Sections~\ref{sec:indist} and
\ref{sec:reducts} into a method for compressing learning problems.

Let $\prob = (G, \train, \loss, \topol)$ be a learning problem. Let $\rho$ be a $d$-substitution. We then define the
reductions $\train/\rho$ and $\loss/\rho$ of  $\train$ and $\loss$ by $\rho$
as follows:
\begin{align*}
  \train/\rho & = \{ \rho(v) \mid v \in \train \} \\
  \loss/\rho(v, c) & = \sum_{w \in \train \cap \class{G,v}{d}} \loss(w,c) 
\end{align*} 
We denote by $\prob/\rho$ the learning problem $(G/\rho, \train/\rho, \loss/\rho, \topol)$.

\begin{theoremrep}
  \label{thm:compression}
  Let $\prob$ be a learning problem with hypothesis space $\topol$ and assume
  that $d \in \natsinf$ is such that every \gnn in $\topol$ has at most $d$
  layers. (In particular, $d = \infty$ if there is no bound on the number of
  layers in $\topol$.) Then $\prob \equiv \prob/\rho$ for every $d$-substitution
  $\rho$.
\end{theoremrep}
\begin{proof}
  Let $\prob/\rho = (G', \train', \loss')$, so $G'= G/\rho$;
  $\train'= \train/\rho$, and $\loss'= \loss/\rho$. Consider an arbitrary but
  fixed \gnn $\L \in \topol$.  We need to show that
  $\loss(\L(G),\train) = \loss'(\L(G'),\train')$. To that end, first observe that for every node
  $v \in G'$ and every node $w \in \class{G,v}{d}$ we have $\rho(w) = v$. Consequently,  
  \begin{align*}
    \colr^d(G',v) = \colr^d(G', \rho(w)) = \colr^d(G/\rho, \rho(w)) = \colr^d(G,w),
  \end{align*}
  Here, the last equality follows from Proposition~\ref{prop:compression-same-colr}. Hence, by Proposition~\ref{prop:indistinguishable-d} (or Corollary~\ref{cor:indistinguishable-inf} when $d = \infty$), $(G',v) \indist{\topol} (G, w)$ and such $\L(G')(v) = \L(G)(w)$ for every $v \in G'$ and  $w \in \class{G,v}{d}$.

  Using this observation we now reason as follows.
  \begin{align*}
    \loss'(\L(G'),\train') & = \sum_{v \in \train'} \loss'(v, \L(G')(v)) \\
                   & = \sum_{v \in \train'} \sum_{w \in \train \cap \class{G,v}{D}} \loss(w, \L(G')(v)) \\
                & = \sum_{v \in \train'} \sum_{w \in \train \cap \class{G,v}{D}} \loss(w, \L(G)(w)) \\
                   & =  \sum_{w\in \train} \loss(w, \L(G)(w)) \\
    & = \loss(\L(G),\train)
  \end{align*}
  The second  equality is by definition of $\loss'$. The third equality
  follows from our observation. The fourth equality is because $\train = \bigcup_{v \in \train'} \class{G,v}{d} \cap T$. 
\end{proof}

\subsection{Graded Color Refinement}
\label{sec:graded}

So far, we have focused on compression based on the \emph{depth} of the \gnns
present in the hypothesis space, where the depth of a \gnn equals its number of
layers and the depth of a hypothesis space $\topol$ is the maximum depth of any
of its \gnns, or $\infty$ if this maximum is unbounded. In particular, the notion of $d$-reduct hinges on this parameter $d$ through the calculated color refinement classes $\class{G,v}{d}$. 

As we have already observed in Section~\ref{sec:indist}, however,
when $\topol$ does not contain all $d$-layer \gnns the color refinement classes
$\class{G,v}{d}$ that we base compression on may contain too few nodes compared
to $\class{G,v}{\topol}$ and may therefore not provide enough opportunity for
compression.

In such cases, it may be beneficial to move to more fine-grained notions of color refinement that better capture $\class{G,v}{\topol}$. In this section we propose one such fine-grained notion, which takes into account  the ``width'' of $\topol$. We say that \emph{\gnn $\L$ has width $c \in \natsinf$} if for every layer in $\L$ the aggregation function $\agg$ is such that $\agg(M) = \agg(\mrestr{M}{c})$ for every multiset $M$. In other words: $\agg$ can only ``count'' up to $c$ copies of a neighbor's color. When $c = \infty$ there is no limit on the count. The width of $\topol$ is then the maximum width of any of its \gnns, or $\infty$  if this maximum is unbounded. Hypothesis spaces of bounded width clearly do not contain \emph{all}  $d$-layer \gnns, for any depth $d$.

While we know of no practical \gnn architectures that explicitly limit the
width, there are influential learning \emph{algorithms} that implicitly limit the width. For
example, to speed up learning, GraphSAGE~\cite{hamiltonInductiveRepresentationLearning2017} can be parametrized by a
hyperparameter $c$. When $c < \infty$ GraphSAGE does not consider all of a
node's neighbors in each layer, but only a random sample of at most $c$ neighbors of that node. Such sampling effectively limits the width of $\topol$.

We next define a variant of color refinement, called \emph{graded} color
refinement, that takes width into account. It may lead to larger color
refinement classes, and therefore potentially also to better compression. \begin{definition}
  \label{def:graded-color-refinement}
  Let $c \in \natsinf$.
  The (one-step) \emph{$c$-graded color refinement} of colored graph
  $G=(V,E,g)$, denoted $\colr_c(G)$, is the colored graph $G'= (V,E,g')$ with 
\[ g'\colon v \mapsto \big( G(v), \mrestr{\mset{ G(w) \mid w \in
    \inc_G(v)}}{c}\big).\]
\end{definition}
Note that $\colr_{\infty}$ equals normal, non-graded, color
refinement. We also remark that $\colr_c$ with $c = 1$ corresponds to standard
\emph{bisimulation} on graphs~\cite{baierPrinciplesModelChecking2008}. We denote by $\colr_c^d(G)$  the result of applying $d$ refinement steps of $c$-graded color refinement, so $\colr_c^0(G) = G$ and $\colr_c^{d+1} = \colr(\colr_c^d(G))$. We denote by $\class{G,v}{c,d}$ the set of all nodes in $G$ that receive the same color after $d$ steps of $c$-graded color refinement. The concept of $(c,d)$-substitution is then defined analogously to $d$-substitution, as mappings from the $\class{G,v}{c,d}$ color refinement classes to nodes in these classes. The reduction of a graph $G$ by a $(c,d)$ substitution is similar to reductions by $d$-substitution except that  the edge multiplicity $E(v \to w)$ for $v,w$ in the reduction is now limited to $c$, i.e.,
\[ E(v\to w) = \sum_{v'\in \class{G,v}{c,d}} \mrestr{G(v' \to w)}{c}. \]
With these definitions, we can show that for any $\prob$ with hypothesis space $\topol$ of width $c$ and depth $d$ we have $\prob \equiv \prob/\rho$ for any $(c,d)$-reduction $\rho$. The full development is omitted due to lack of space.

We note that $\class{G,v}{d} \subseteq \class{G,v}{c,d}$, always. Graded color
refinement hence potentially leads to better compression, but only applies to
problems with hypothesis spaces of width $c$.  We will empirically contrast the
compression ratio obtained by $(c,d)$-reducts to those obtained by $d$-reducts
(i.e., where $c = \infty$) in Section~\ref{sec:evaluation}. There, we also study
the effect of $c$ on learning accuracy for problems whose width is not bounded.

\section{Evaluation}
\label{sec:evaluation}
In this section, we empirically evaluate the compression methodology described in Section~\ref{sec:methodology}. We first give insights into the compression ratios that can be obtained on real-world graphs in Section~\ref{sec:compression-evaluation}. Subsequently, we evaluate the  learning on compressed graphs versus learning on the original graphs.

\subsection{Compression}
\label{sec:compression-evaluation}

\newcommand{\arxiv}{\textsf{ogbn-arxiv}\xspace}
\newcommand{\arxivinv}{\textsf{ogbn-arxiv-inv}\xspace} 
\newcommand{\arxivun}{\textsf{ogbn-arxiv-undirected}\xspace}
\newcommand{\products}{\textsf{ogbn-products}\xspace}
\newcommand{\roadnetca}{\textsf{snap-roadnet-ca}\xspace}
\newcommand{\roadnetpa}{\textsf{snap-roadnet-pa}\xspace}
\newcommand{\roadnettx}{\textsf{snap-roadnet-tx}\xspace}
\newcommand{\socpokec}{\textsf{snap-soc-pokec}\xspace}

We consider the real-world graphs listed in
Table~\ref{tab:compression-datasets}. The $\textsf{ogbn-*}$
graphs are from the Open Graph Benchmark (OGB), a collection of
realistic, large-scale, and diverse benchmark datasets for
machine learning on graphs~\cite{huOpenGraphBenchmark2020}, where
they belong to the OGB node property prediction benchmark (OGBN).
Specifically, \arxiv is a network of academic papers, where edge
$x \to y$ indicates that $x$ cites $y$. Graph \arxivinv is the
inverted version of \arxiv; it is obtained by reversing edges, so
that edge $x \gets y$ indicates that $y$ was cited by $x$. Graph
\arxivun is the undirected version of \arxiv; it is obtained by
adding inverse edges to \arxiv. Next, there is \products, an undirected graph
representing an Amazon product co-purchasing network.  The other
datasets are from the Stanford Large Network Dataset Collection
(SNAP)~\cite{leskovecSNAPDatasetsStanford}. Here, \roadnetca,
\roadnetpa, and \roadnettx are undirected graphs representing
road networks, and \socpokec is a directed graph containing an
online social network.

\begin{table}[tbp]
  \small
  \begin{tabular}[h]{@{}llrr@{}}
    \toprule Graph & Type      & \#Nodes    & \#Edges    \\
    \midrule
    \begin{filecontents*}{data/compression_dataset_descriptions.csv}
dataset,type,vertices,edges
{\arxiv},directed,169.343,1.166.243
{\arxivinv},directed,169.343,1.166.243
{\arxivun},undirected,169.343,2.315.598
{\products},undirected,2.449.029,61.859.140
{\roadnetca},undirected,1.965.206,5.533.214
{\roadnetpa},undirected,1.088.092,3.083.796
{\roadnettx},undirected,1.379.917,3.843.320
{\socpokec},directed,1.632.803,30.622.564
\end{filecontents*}
\csvreader[
    late after line=                                     \\,
    late after last line=                                \\\bottomrule,
    ]
    {data/compression_dataset_descriptions.csv}{}%
    {\csvcoli      & \csvcolii & \csvcoliii & \csvcoliv} %
  \end{tabular}
  \vspace{1ex}
  \caption{Datasets}
  \label{tab:compression-datasets}
  \vspace{-2ex}
\end{table}

The input colors used for learning on these graphs typically depend on the application. To get an understanding of the maximum amount of compression that we can obtain independent of the target application, we assign a shared single color $c$ to each node, in all graphs. As such, the color refinement classes $\class{G,v}{d}$ that we obtain are maximal, in the sense that for any other colorored graph $G'$ whose topology equals $G$ we will have $\class{G',v}{d} \subseteq \class{G,v}{d}$. In this setting, we hence reach maximal compression.

We consider three versions of \arxiv to get an indication of how edge directionality impacts compression. Recall that \gnn layers propagate color information following the direction of edges. Hence, because in \arxiv an edge $x \to y$ indicates that $x$ cites $y$, color information flows from citing papers to cited papers. In \arxivinv, by contrast, it flows from cited papers to citing papers while in \arxivun it flows in both directions. The direction in which the information can flow impacts the number of color refinement classes that we obtain, and hence the compression, as we will see.

\begin{figure*}[tpb]
  \small
  \pgfplotsset{
    width=5.2cm,
    every axis plot/.append style={thick},
    every axis y label/.style={
        at={(ticklabel cs:0.5)},rotate=90,anchor=near ticklabel,
      },
    every axis x label/.style={
        at={(ticklabel cs:0.5)},anchor=near ticklabel,
      },
    ymin=-0.05, ymax=1.03,
    ymajorgrids=true,
    xmajorgrids=true,
    legend columns=-1,
    legend style={draw=none},
  }
  \pgfplotstableread{data/compression-ungraded-nodes.tsv}\ungradnodes

  \begin{tikzpicture}[baseline]
    \begin{axis}[
        title style={align=center},
        title={$(a)$ $\%$ of nodes in $d$-reduct \\ relative to original},
        xlabel={Refinement depth $d$},
        xmin=-0.0, xmax=7.09,
        cycle list name=exotic,
        legend columns=4, 
        legend to name=namedlegend,
      ]
      \addplot
      table [x=round,y=unlab-ogbn-arxiv]  {\ungradnodes};
      \addlegendentry{\arxiv};

      \addplot
      table [x=round,y=unlab-ogbn-arxiv-inv]
        {\ungradnodes};
      \addlegendentry{\arxivinv};

      \addplot
      table [x=round,y=unlab-ogbn-arxiv-undirected]
        {\ungradnodes};
      \addlegendentry{\arxivun};

      \addplot
      table [x=round,y=unlab-ogbn-products] {\ungradnodes};
      \addlegendentry{\products};

      \addplot
      table [x=round,y=unlab-snap-roadnet-ca] {\ungradnodes};
      \addlegendentry{\roadnetca};

      \addplot
      table [x=round,y=unlab-snap-roadnet-pa] {\ungradnodes};
      \addlegendentry{\roadnetpa};

      \addplot
      table [x=round,y=unlab-snap-roadnet-tx] {\ungradnodes};
      \addlegendentry{\roadnettx};

      \addplot
      table [x=round,y=unlab-snap-soc-pokec] {\ungradnodes};
      \addlegendentry{\socpokec};
    \end{axis}
  \end{tikzpicture}
  \hskip 4pt 
  \begin{tikzpicture}[baseline]
    \begin{axis}[
        title style={align=center},
        title={$(b)$ $\%$ of edges in $d$-reduct \\ relative to original},
        xlabel={Refinement depth $d$},
        xmin=-0.0, xmax=7.09,
        cycle list name=exotic,
      ]
      \pgfplotstableread{data/compression-ungraded-edges.tsv}\ungradedges
      \addplot
      table [x=round,y=unlab-ogbn-arxiv]  {\ungradedges};

      \addplot
      table [x=round,y=unlab-ogbn-arxiv-inv]
        {\ungradedges};

      \addplot
      table [x=round,y=unlab-ogbn-arxiv-undirected]
        {\ungradedges};

      \addplot
      table [x=round,y=unlab-ogbn-products] {\ungradedges};

      \addplot
      table [x=round,y=unlab-snap-roadnet-ca] {\ungradedges};

      \addplot
      table [x=round,y=unlab-snap-roadnet-pa] {\ungradedges};

      \addplot
      table [x=round,y=unlab-snap-roadnet-tx] {\ungradedges};

      \addplot
      table [x=round,y=unlab-snap-soc-pokec] {\ungradedges};
    \end{axis}
  \end{tikzpicture}
  \hskip 4pt 
  \begin{tikzpicture}[baseline]
    \begin{axis}[
        title style={align=center},
        title={$(c)$ $\%$ of nodes in $(c,3)$ reduct \\ relative to $3$-reduct},
        xlabel={Grade $c$},
        xmin=1.0, xmax=5.09,
        cycle list name=exotic,
      ]
      \pgfplotstableread{data/compression-graded-nodes-rel.tsv}\gradednodes
      \addplot
      table [x=grade,y=unlab-ogbn-arxiv]  {\gradednodes};

      \addplot
      table [x=grade,y=unlab-ogbn-arxiv-inv]
        {\gradednodes};

      \addplot
      table [x=grade,y=unlab-ogbn-arxiv-undirected]
        {\gradednodes};

      \addplot
      table [x=grade,y=unlab-ogbn-products] {\gradednodes};

      \addplot
      table [x=grade,y=unlab-snap-roadnet-ca] {\gradednodes};

      \addplot
      table [x=grade,y=unlab-snap-roadnet-pa] {\gradednodes};

      \addplot
      table [x=grade,y=unlab-snap-roadnet-tx] {\gradednodes};

      \addplot
      table [x=grade,y=unlab-snap-soc-pokec] {\gradednodes};
    \end{axis}
  \end{tikzpicture}
  \hskip 4pt 
  \begin{tikzpicture}[baseline]
    \begin{axis}[
        title style={align=center},
        title={$(d)$ $\%$ of edges in $(c,3)$ reduct \\ relative to $3$-reduct},
        xlabel={Grade $c$},
        xmin=1.0, xmax=5.09,
        cycle list name=exotic,
      ]
      \pgfplotstableread{data/compression-graded-edges-rel.tsv}\gradededges
      \addplot
      table [x=grade,y=unlab-ogbn-arxiv]  {\gradededges};

      \addplot
      table [x=grade,y=unlab-ogbn-arxiv-inv]
        {\gradededges};

      \addplot
      table [x=grade,y=unlab-ogbn-arxiv-undirected]
        {\gradededges};

      \addplot
      table [x=grade,y=unlab-ogbn-products] {\gradededges};

      \addplot
      table [x=grade,y=unlab-snap-roadnet-ca] {\gradededges};

      \addplot
      table [x=grade,y=unlab-snap-roadnet-pa] {\gradededges};

      \addplot
      table [x=grade,y=unlab-snap-roadnet-tx] {\gradededges};

      \addplot
      table [x=grade,y=unlab-snap-soc-pokec] {\gradededges};
    \end{axis}
  \end{tikzpicture}

  \begin{center}
    \ref{namedlegend}
  \end{center}
  \vspace{-2ex}
  \caption{Normalized reduction in nodes and edges using $d$-reduction (a  and b) and $(c,3)$-reduction (c and d).  \label{fig:compression}}
  \vspace{2ex}
\end{figure*}
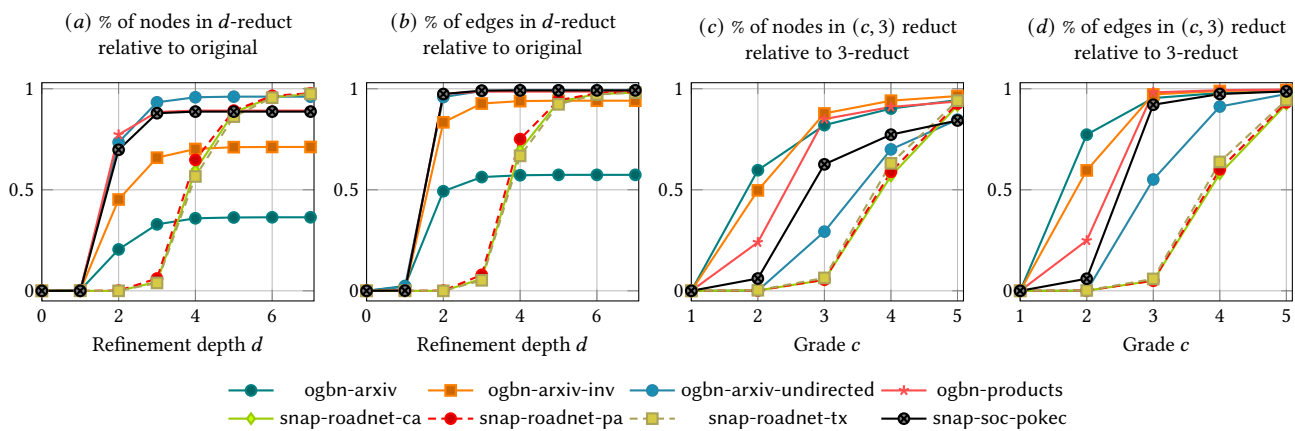

\paragraph*{Ungraded compression}
Figures~\ref{fig:compression}(a) and \ref{fig:compression}(b) shows the fraction of nodes and edges remaining in $d$-reducts of these graphs, plotted as a function of the number of color refinement rounds $d$. (Consistent with our definition of the size of a multigraph, the number of edges plotted is the number of unique edges, i.e., ignoring edge multiplicities.) We see that in terms of the number of nodes, compression is effective for $d \leq 2$, obtaining compression ratios of $0.03\%-0.05\%$ on the road network datasets, to $77\%$ on \products. For $d = 3$, compression becomes ineffective for \arxivun, \products and \socpokec as these retain $88\%$ of their nodes or more. Compression on the other datasets is satsifactory for $d=3$, as shown in Table~\ref{tab:compression-d-3}. From $d > 3$ onwards, node compression becomes ineffective for most datasets, with all datasets except \arxiv retaining at least $86\%$ of their nodes when $d \geq 5$. By contrast, \arxiv stabilizes at $d = 4$, retaining only $36\%$ of its nodes.

In terms of the number of edges, we see that \arxivinv, \arxivun, \products, and \socpokec retain almost $100\%$ of their edges from $d \geq 2$ onwards, while \arxiv retains only $56\%$ of its edges when $d=3$ and the road networks compress to $5\%-8\%$ of the edges when $d = 3$.

While we may hence conclude that in general compression becomes ineffective for deeper layer numbers, $d > 4$, we note that in practice most \gnn topologies use only $d = 3$ layers. For such \gnns, compression on \arxiv and the road networks is promising.

One may wonder why the inverted and undirected variants of \arxiv
differ so much in terms of compression. The short answer is that
that, their graph topology is completely different. As such, the
local neighborhood information that $\colr^d$ calculates is also
completely different, yielding different reductions. For example,
a manual inspection of \arxiv reveals a number of highly cited papers that have no outgoing edges. While these papers are the ``sink nodes'' in \arxiv, they are ``source nodes'' in \arxivinv. Quite quickly we may then distinguish nodes in \arxivinv based solely on the number of highly cited papers that they cite. This behavior does not occur in \arxiv, because the highly cited papers are outgoing neighbors, which $\colr^d$  ignores.

\begin{table}[tpb]
  \small
  \begin{tabular}[h]{@{}lrr@{}}
    \toprule Graph & Nodes ($\%$) & Edges ($\%$) \\
    \midrule
    \begin{filecontents*}{data/compression-ungraded-depth-3.csv}
dataset,pct_nodes,pct_edges
{\arxiv},32.9,56.3
{\arxivinv},65.9,92.7
{\arxivun},93.3,98.9
{\products},88.5,98.6
{\roadnetca},4.4,5.9
{\roadnetpa},6.1,7.9
{\roadnettx},3.9,5.2
{\socpokec},87.9,99.1
\end{filecontents*}
\csvreader[
    late after line=                             \\,
    late after last line=                        \\\bottomrule,
    ]
    {data/compression-ungraded-depth-3.csv}{}%
    {\csvcoli      & \csvcolii    & \csvcoliii}  %
  \end{tabular}
  \vspace{1ex}
  \caption{Compression at $d=3$.}
  \vspace{-2ex}
  \label{tab:compression-d-3}
\end{table}

\paragraph*{Graded compression} We next evaluate the effect of moving to
compression based on $c$-graded color
refinement. Figure~\ref{fig:compression}(c) and (d) shows the fraction of nodes
and edges remaining in $(c,3)$-reducts of our graphs, relative to the number of nodes and edges in the corresponding $d=3$-reduct, plotted as a
function of $c$. In terms of the number of nodes we see that, as expected, moving to graded compression yields better compression than non-graded compression. In particular, for $c=1$ all datasets retain $< 1\%$ of their nodes while for $c =3$, compression is $27\%$ or less for \arxiv, \arxivun, and the road networks. In the latter setting, \products is at $75\%$ of nodes and \socpokec at $55\%$. Compression hence becomes less effective as $c$ increases.

In terms of the number of edges, compression is most effective when $c = 1$ or $c=2$ in which case it significantly improves over ungraded compression. From $c \geq 3$ onwards, graded compression becomes only marginally better than ungraded compression.

We conclude that graded compression has the potential to yield significantly smaller graphs than ungraded compression, but only for small gradedness values, $c=1$ (i.e., bisimulation) or $c=2$.

\paragraph*{Conclusion}
Overall, we see that real-world graphs have  diverse non-graded compression ratios: road networks compress extremely well (to $4\%$ of nodes and $5\%$ of
  edges when $d = 3$); while other graphs such as \arxivinv compress reasonably in terms of number of nodes ($65\%$) but only marginally in terms of edges ($93\%$ or more); and graphs such as \products compress only marginally in both (retaining $90\%$ or more of nodes and edges).
This diversity is to be expected: our exact compression methodology is based on exploiting redundancy in local neighborhoods of nodes. For graphs where few nodes have equal local neighborhoods, we cannot expect reduction in size. Moving to graded compression for those graphs improves reduction in size, but will yield approximate compression unless the learning problem hypothesis space has small width. 

\subsection{Learning}
\label{sec:learning-evaluation}

\begin{table}[tpb]
  \small
  \begin{tabular}[h]{@{}crrccc@{}}
    \toprule Problem & Nodes                      & Edges                      & Accurracy                  & \multicolumn{2}{c}{Training}               \\
                     & \multicolumn{1}{c}{($\%$)} & \multicolumn{1}{c}{($\%$)} & \multicolumn{1}{c}{($\%$)} & time ($s$)                   & mem (GiB)   \\
    \midrule
    \begin{filecontents*}{data/learning-summary.csv}
problem,c,pct_nodes,pct_edges,pct_test_accurracy,training_time,memory
{$\prob_1$},,100.0,100.0,66.7,61.66,2.14
{$\prob_2$},,100.0,100.0,65.9,58.99,2.02
{$\prob_3^1$},1,76.9,95.9,60.9,48.10,1.59
{$\prob_3^3$},3,79.2,97.1,64.5,49.71,1.65
{$\prob_3^5$},5,79.3,97.2,64.9,49.55,1.65
{$\prob_3^{\infty}$},{$\infty$},79.4,97.2,63.8,49.16,1.65
\end{filecontents*}
\csvreader[
    late after line=                                                                                                                                     \\,
    late after last line=                                                                                                                                \\\bottomrule,
    ]
    {data/learning-summary.csv}{}%
    {\csvcoli        & \csvcoliii                 & \csvcoliv                  & \csvcolv                   & \csvcolvi                    & \csvcolvii} %
  \end{tabular}
  \vspace{1ex}
  \caption{Comparison of learning on the original uncompressed problem $\prob_1$, the uncompressed problem with discretized labels $\prob_2$, and $(c, d=3)$-compressed variants $\prob_3^c$.   \label{tab:compressed-learning}}
  \vspace{-3ex}

\end{table}

We next turn to validating empirically the effect of learning on compressed
problems according to our methodology. Specifically, we apply our methodology to learning on \arxivinv with $d = 3$.  We know from Section \ref{sec:compression-evaluation} that  compression on this graph is reasonable in terms of nodes, but only marginal in terms of edges. Despite the modest compression, the effect on learning efficiency in terms of learning time and memory consumption is still positive, as we will see.

We hence focus in this section on
\arxivinv and the associated learning problem from the OGBN benchmark: predict,
for every paper in \arxivinv, its subject area (e.g., cs.AI, cs.LG, and cs.OS,
\dots). There are $40$ possible subject areas. In addition to the citation
network, we have available for each paper a $128$-dimensional feature vector
obtained by averaging the embeddings of words in its title and abstract. Our
learning problem $\prob_1$ hence consists of the \arxivinv graph in which each
node is colored with this feature vector. The training set $T$ consists of 90941
nodes, obtained conform the OGBN benchmark. The hypothesis space $\topol$
consists of \gnns that all share the same topology and  vary only in their
concrete parameters. The topology consists of $3$ GNN layers whose $\agg$
function computes the mean of all colors of incoming neighbors. The $\comb$
function applies a linear transformation to $G(v)$, applies a linear transformation to the result of the aggregation, and finally sums these two intermediates together. Each layer, except the final one, is followed by a batch normalisation as well ReLU to introduce non-linearity. The layers have dimensions $(128,256)$, $(256,256)$ and $(256, 40)$, respectively. We apply a $50\%$ dropout between layers during training. 
Softmax is applied after the last layer to turn feature vectors into subject
areas, and we use cross-entropy as loss function.

Unfortunately, the initial coloring in $\prob_1$ is such that every node has a distinct color. Therefore, every node is in its own unique color refinement class, and compression is not possible. We therefore transform $\prob_1$ into a problem $\prob_2$ that can be compressed by first converting the $128$-dimensional word embedding vectors into estimates of paper areas by means of a multilayer perceptron (MLP) that is trained on the nodes in $T$ \emph{but without having the graph structure available}. This hence yields an initial estimate of the  paper area for each node. By learning a \gnn on the graph that is colored by one-hot encodings of these initial estimates, the estimates can be refined based on the graph topology. We denote the new problem hence obtained by $\prob_2$. Note that $\prob_2$ has the same training set, hypothesis space, and loss function as $\prob_1$. The MLP has input dimension $128$, one hidden layer of dimension 256, and output layer of dimension $40$.

We next compress $\prob_2$ using $(c,d)$-reduction with $d=3$. The resulting compressed problems are denoted $\prob_3^c$. The corresponding compressed graph sizes are shown in Table~\ref{tab:compressed-learning}. We note that, because we consider labeled graphs here, the compression ratio is worse than the maximum-compression scenario of Section~\ref{sec:compression-evaluation}.

We gauge the generalisation power of the learned \gnns by computing the accurracy on the test set, determined by the OGBN benchmark, comprising 48603 nodes. We learn for 256 epochs with learning rate $0.01$ on all problems. All experiments are run on an HP ZBook Fury G8 with Intel Core i9 11950H CPU, 32 GB of RAM and NVIDIA RTX A3000 GPU with 6 GB RAM.

The results are summarized in Table~\ref{tab:compressed-learning}. Comparing $\prob_1$ with $\prob_2$ we see that estimating the paper area through an MLP has marginal effect on the test accurracy, training time and memory consumption. Further comparing $\prob_2$ with $\prob_3^c$, which are equivalent by Theorem~\ref{thm:compression}, we see that training accurracy is indeed comparable between $\prob_2$ and $\prob_3^{\infty}$; we attribute the difference in accurracy to the stochastic nature of learning. There is a larger difference in accurracy between $\prob_2$ and $\prob_3^{c}$ when $c = 1$, i.e., when we compress based on bisimulation, than when $c > 1$. This is because $c$-graded compression is an approximation, as explained in Section~\ref{sec:graded}, and because, as we can see, $c$-graded compression for $c > 1$ is nearly identical in size to ungraded compression. For $c > 1$ we may hence expect there to be only marginal differences w.r.t. ungraded compression. Learning the compressed problems  $\prob_3^c$ is more efficient than learning on the uncompressed $\prob_2$, taking only $81.5$--$84.3\%$ of the learning time and $78.7$--$81.6\%$ of the memory respectively which is comparable to the reduction in number of nodes when compressing.

Our evaluation in this section  is on a single learning problem; it should hence be interpreted as a preliminary insight that requires further evaluation. Based on these preliminary findings, however, we conclude that compressed learning
can yield observable improvements in training time and memory consumption.

\section{Conclusion and Future Work}
\label{sec:conclusion}
We have proposed a formal methodology for exact compression of \gnn-based learning problems. While the attainable exact compression ratio depends on the input graph, our experiment in Section~\ref{sec:learning-evaluation} nevertheless indicates that observable improvements in learning efficiency are possible, even when the compression in terms of the number of edges is negligible.

In terms of  future work, first and foremost our
preliminary empirical evaluation should be extended to more learning tasks.
Second, as we have seen $c$-graded color refinement offers a principled way of
approximating exact compression, which becomes exact for hypothesis spaces of width $c$.  It
is an interesting question whether other existing approximate compression
proposals~\cite{dengGraphZoomMultilevelSpectral2020a,liangMILEMultiLevelFramework2021}
can similarly be tied to structural properties of the hypothesis
space. 

\begin{acks}
We thank Floris Geerts for helpful discussions.
S. Vansummeren and J. Steegmans were supported by the Bijzonder Onderzoeksfonds (BOF) of Hasselt University under Grants No. BOF20ZAP02 and BOF21D11VDBJ.  This work was further supported by the Research Foundation Flanders (FWO) under research project Grant No. G019222N. We acknowledge computing resources and services provided by the VSC (Flemish Supercomputer Center), funded by the Research Foundation – Flanders (FWO) and the Flemish Government.
\end{acks}

\bibliographystyle{ACM-Reference-Format}
\bibliography{zotero_gnn}

\end{document}